\documentclass[journal]{IEEEtran}
\ifCLASSINFOpdf
\else
\fi
\usepackage{color}
\usepackage{epsfig}
\usepackage{graphicx}
\usepackage{amsmath}
\usepackage{amssymb}
\usepackage{bm}
\usepackage{algorithm}
\usepackage{url}
\usepackage{algorithmic}
\usepackage{multirow}
\usepackage{subfigure}

\usepackage{amsthm}
\newtheorem{thm}{Theorem}

\begin{document}
%
\title{Learning Aggregated Transmission Propagation Networks for Haze Removal and Beyond}
%
%
%

\author{Risheng~Liu,~\IEEEmembership{Member,~IEEE,} Xin~Fan,~\IEEEmembership{Member,~IEEE,}
        Minjun~Hou, Zhiying~Jiang, Zhongxuan~Luo
        and~Lei~Zhang,~\IEEEmembership{Fellow,~IEEE}
\thanks{This work is partially supported by the National Natural Science Foundation of China (Nos. 61672125, 61300086, 61572096, and 61432003), the Fundamental Research Funds for the Central Universities and the Hong Kong Scholar Program.}
\thanks{R. Liu is with DUT-RU International School of Information Science \& Engineering and Key Laboratory for Ubiquitous Network and Service Software of Liaoning Province, Dalian University of Technology, Dalian, China. (Corresponding author, e-mail: rsliu@dlut.edu.cn).}
\thanks{X. Fan, M. Hou, Z. Jiang and Z. Luo are with DUT-RU International School of Information Science \& Engineering, and Key Laboratory for Ubiquitous Network and Service Software of Liaoning Province, Dalian University of Technology, Dalian, China.}
\thanks{L. Zhang is with Department of Computing, the Hong Kong Polytechnic University.}

\thanks{Manuscript received April 19, 2018; revised August 26, 2018.}
}

%
%

\markboth{Journal of \LaTeX\ Class Files,~Vol.~14, No.~8, August~2018}%
{Shell \MakeLowercase{\textit{et al.}}: Bare Demo of IEEEtran.cls for IEEE Journals}
%



\maketitle

\begin{abstract}
Single image dehazing is an important low-level vision
task with many applications. Early researches have investigated different kinds
of visual priors to address this problem. However,
they may fail when their assumptions are not valid on specific images.
Recent deep networks also achieve relatively good performance in this task.
But unfortunately, due to the disappreciation of rich physical rules in hazes,
large amounts of data are required for their training. More importantly,
they may still fail when there exist completely different haze distributions in testing images.
By considering the collaborations of these two perspectives,
this paper designs a novel residual architecture to aggregate
both prior (i.e., domain knowledge) and data (i.e., haze distribution)
information to propagate transmissions for scene radiance estimation. We
further present a variational energy based perspective to investigate the intrinsic
propagation behavior of our aggregated deep model. In this way, we actually
bridge the gap between prior driven models and data driven networks and leverage
advantages but avoid limitations of previous dehazing approaches. A lightweight
learning framework is proposed to train our propagation network.
Finally, by introducing a task-aware image separation formulation with a flexible optimization scheme, we extend the proposed model for more challenging vision tasks, such as underwater image enhancement and single image rain removal.
Experiments on both synthetic and real-world images demonstrate the
effectiveness and efficiency of the proposed framework.
\end{abstract}


%
\IEEEpeerreviewmaketitle

\section{Introduction}

\begin{figure}[!htbp]
	\centering
		\begin{tabular}{c@{\extracolsep{0.3em}}c@{\extracolsep{0.3em}}
		c@{\extracolsep{0.3em}}c}
	\includegraphics[width=.24\textwidth]{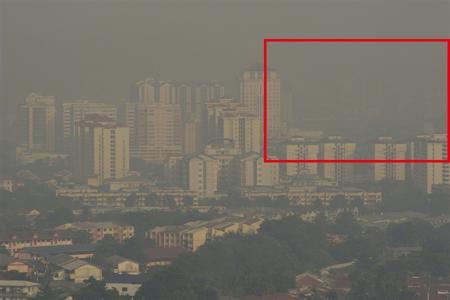}
	&\includegraphics[width=.24\textwidth]{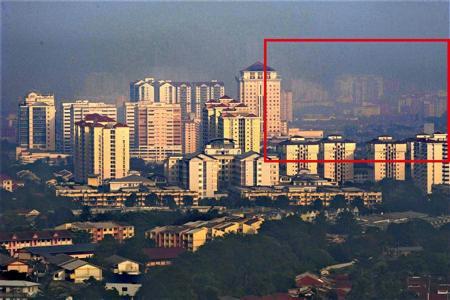}\\
	Image with dense hazes & Ours\\
	\includegraphics[width=.24\textwidth]{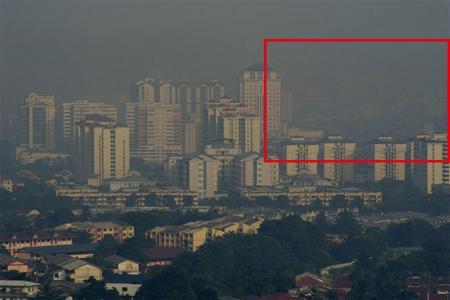}
	&\includegraphics[width=.24\textwidth]{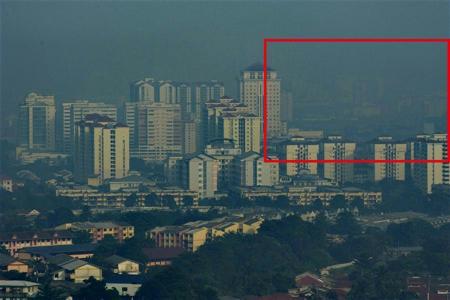}\\
	Cai~\emph{et al.} & Ren~\emph{et al.}
\end{tabular}
	\caption{Single image dehazing results by different propagation networks on a real-world example. The zoomed in comparisons of red rectangle regions are shown in supplemental materials.}\label{fig:first}
\end{figure}
\IEEEPARstart{H}aze is a slight obscuration of the lower atmosphere, typically caused by fine suspended particles. Due to this atmospheric phenomenon,
images captured by camera will fade and the contrast of the observed object will be reduced.
Mathematically, hazy image can be formulated as a per-pixel convex combination of the haze-free image and the global atmospheric light. Image dehazing is
just the task of recovering haze-free images from hazy observations, which has wide applications in
computer vision society (e.g., outdoor surveillance, object detection, remote sensing and underwater photograph, just name a few).

It has been investigated that image degradations caused by hazes will increase with the distance from the camera, since the scene radiance decreases and the atmospheric light magnitude increases. Thus haze removal
mainly relies on the scene depth information. But unfortunately, estimating
depths of pixels in the image is also a challenging task. Early approaches~\cite{schechner2001instant} often
utilize multiple images of the same scene to deal with this problem. However, in real-world
scenario, there may only exist one single image for a specific scene. Thus recent
literatures tend to address the more challenging single image dehazing task, which
has even fewer scene information~\cite{xu2016review}.



\subsection{Related Works}


We first survey some related works on single image dehazing, which can be roughly divided into two categories, i.e., prior driven
models and data driven networks.

\textbf{Prior Driven Models:} In past years, various priors (also known as domain knowledge based cues) have been investigated to capture deterministic (e.g., physical rules and variational energy) or statistical (e.g., distributions) properties of hazy images.
For instance, Fattal~\cite{Fattal2008} estimated depth maps and haze-free images based on the assumption that surface shadings and transmissions should be
locally uncorrelated.
He~\emph{et~al.}~\cite{he2011single} observed that
most local patches in haze-free images contain extremely low
intensities in at least one color channel and then
the thickness of the haze can be estimated by the so-called dark channel prior (DCP).
The works in \cite{Tarel2009,Gibson2012}
improved DCP by replacing the minimum operator with the median operator. In~\cite{meng2013efficient}, Meng~\emph{et al.} reformulated DCP as a boundary constraint and incorporated a contextual regularization to model transmissions of foggy images. Fan~\emph{et al.}~\cite{fan2016two} proposed a two-layer Gaussian process regression model to refine DCP.
 Lai~\emph{et al.}~\cite{lai2015single} proposed
two scene priors to estimate the optimal transmission map.
By creating a linear formulation for the scene depth under
a color attenuation prior, Zhu~\emph{et~al.}~\cite{zhu2015fast} designed a
supervised haze removal model. Berman~\emph{et~al.}~\cite{berman2016non} presented the concept of haze-line and adopted it to estimate the transmission factors.
The work in~\cite{li2014contrast,chen2016robust} aimed to solve variational models to suppress artifacts in hazy images.

It can be observed that all above mentioned conventional prior driven dehazing methods heavily rely on \emph{certain assumptions of hazy images}. So they may fail on specific data when these cues are not valid. For example, the method of Fattal performs well on thin hazy images. But this method cannot successfully recover images with thick hazes since it needs rich color information and color difference among pixels. And DCP may break down in bright areas of the scene (e.g., sky regions). As for the method of Berman~\emph{et~al.}, it may not perform well when the airlight is significantly brighter than the scene. Though suppressing artifacts, variational models often
need to perform time consuming iterations to obtain the (local) optimal solutions.

\textbf{Data Driven Networks:} Very recently, training heuristically constructed deep convolutional neural networks (CNNs) on large-scale datasets
have delivered record breaking performance in many vision and recognition tasks~\cite{lecun2015deep,he2015deep}.
Some existing methods suggest that CNNs can also benefit single image dehazing and are effective
to handle complex scenes by accurate depth information estimation.
Cai~\emph{et~al.}~\cite{cai2016dehazenet} proposed a trainable end-to-end CNN, which inputs hazy image,
and outputs corresponding transmission map, to recover haze-free image. Ren~\emph{et~al.}~\cite{ren2016single} further
introduced multi-scale CNNs for single image dehazing.
They first used a coarse-scale network to roughly estimate the transmission structure and
then refined it by another fine-scale network.

Although relatively good results have been achieved by existing deep networks,
there are still two major issues of these fully data driven models should be addressed. Firstly, domain knowledge of hazy images, which is shown to be effective in previous works, are completely discarded by these networks. Secondly, the performance of current network architectures are tightly dependent on the quality and quantity of the training data. So
large amounts of data are required for networks training. For example,
the training sets required in \cite{cai2016dehazenet} and \cite{ren2016single} are all of the order of ten thousands. It can also be seen that even with amounts of training images, standard deep models may still
fail on images with different haze distributions due to their \emph{fully-training-data-dependent} nature.

\subsection{Contributions}
As prior driven dehazing models are mainly derived from particular visual cues,
they may fail when their assumptions are not valid on specific images. On the other hand, recent deep networks tend to learn dehazing processes fully based on training data. So they indeed completely discard physical principles of hazes.
To mitigate above issues, we propose a novel propagation formulation, named
data-and-prior-aggregated transmission network (DPATN), to leverage advantages but avoid limitations of domain knowledge and training data for single image dehazing.
Specifically,
we first build an aggregated residual architecture to formulate transmission propagation.
The intrinsic propagation behavior of DPATN
is then investigated by an energy-based-modeling perspective~\cite{lecun2006tutorial}.
A lightweight learning framework is also established for network training. Figure~\ref{fig:first} shows performances
of DPATN together with two recently proposed CNNs on a challenging real-world example. Finally, we also extend DPATN for more challenging tasks, such as underwater image enhancement and single image rain removal.
In summary, our DPATN framework has at least three advantages over existing dehazing architectures.
\begin{itemize}
	\item Specific domain knowledge of the haze removal task
	is explicitly exploited to the transmission propagation in DPATN, yielding more accurate transmission estimations
	and realistic dehazing results.
	\item DPATN needs extremely less training images
	than existing dehazing networks\footnote{As shown in Section~\ref{sec:data}, dozens of images are enough for
		our DPATN training. While existing dehazing networks (e.g., Cai~\emph{et al.} and Ren~\emph{et al.}) all require tens of thousands of training images.}. Furthermore, in contrast to conventional feed forward networks, DPATN is able to directly propagate
	transmissions through the network in both forward and backward passes, leading to nice propagation properties.
	\item Different from most current deep formulations, which build their architectures in heuristic ways, we provide a novel manner to investigate the intrinsic propagation behaviors of our architectures
	following an energy minimization perspective. These insights can also be extended to guide other complex vision tasks.
	\item To address more challenging image enhancement tasks, such as underwater image enhancement and single image rain removal, we extend DPATN with a task-aware separation formulation and prove the global convergence property of our final propagation.
\end{itemize}




\section{The Proposed Framework}

We start with designing data-and-prior-aggregated transmission network (DPATN), to bridge the gap between domain knowledge and training data for haze removal.

\subsection{Atmospheric Scattering Model}

To formulate hazy images, we first describe the following widely known atmospheric
scattering model~\cite{mccartney1976optics}:
\begin{equation}
	\left\{\begin{array}{l}
		\mathbf{I}(x)=\mathbf{J}(x)t(x) + \mathbf{A}(1-t(x)),\\
		t(x)=\exp(-\beta d(x)),
	\end{array}\right.\label{eq:haze}
\end{equation}
where $\mathbf{A}$ is the global atmospheric light, $\mathbf{I}(x)$, $\mathbf{J}(x)$, $d(x)$ and $t(x)$ are respectively the observed hazy image, the latent scene radiance, the scene depth and the medium transmission at pixel location $x$ and $\beta$ is the medium extinction coefficient.
It is known that the value of $t(x)$ is in the range $[0, 1]$ and used to describe the portion of light that is not scattered and
reaches the camera sensors.
The second equality further indicates that $t(x)$ is exponentially attenuated with the scene depth.
The main goal of image dehazing is to recover $\mathbf{J}(x)$ from $\mathbf{I}(x)$. Following Eq.~\eqref{eq:haze},
we have that estimating accurate transmission map $t(x)$ plays the core role in this task.
However, due to multiple solutions exist for a single hazy image, the problem is highly ill-posed.
Notice that to facilitate the presentation, hereafter we denote the discrete form of $t(x)$ as $\mathbf{t}=[t_1,\cdots,t_N]\in\mathbb{R}^N$, where $N$ is the number of pixels in the hazy image.

\begin{figure*}
	\begin{center}
		\includegraphics[width=1\textwidth]{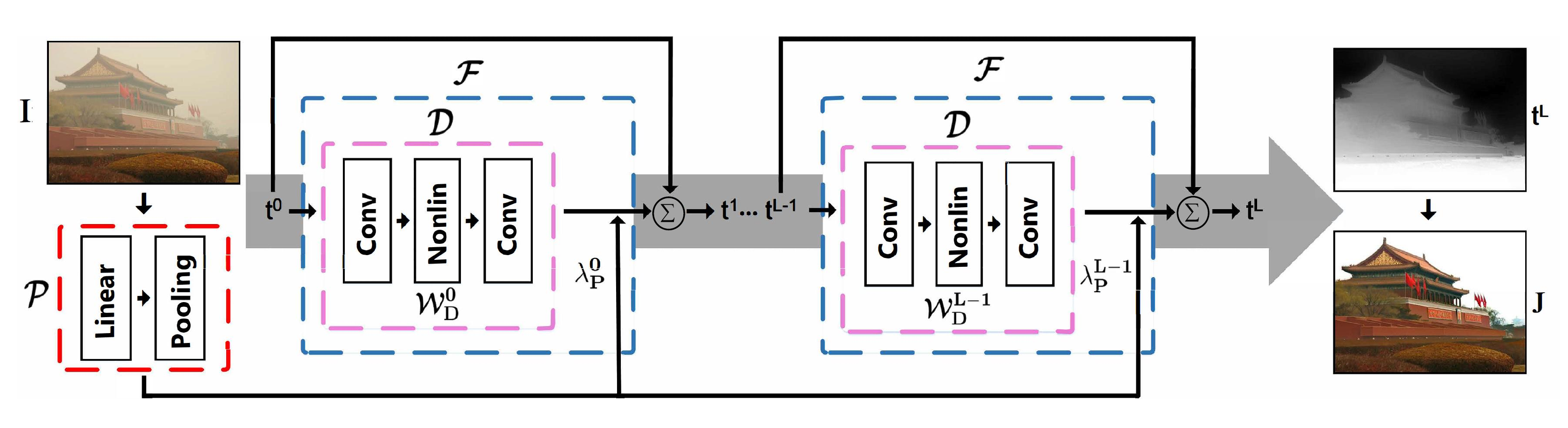}
	\end{center}
	\caption{The overview of transmission propagation in DPATN (along the arrow direction). The red, pink and blue dashed rectangles denote the prior driven (i.e., $\mathcal{P}$), data driven (i.e., $\mathcal{D}$) and aggregated (i.e., $\mathcal{F}$) architectures, respectively. The right-most column
		shows the estimated transmission map and the restored image of DPATN, respectively.}\label{fig:short}
\end{figure*}

\subsection{Aggregated Architecture for Transmission Propagation}

We first design the general propagation scheme of transmission $\mathbf{t}$ as the following residual architecture:
\begin{equation}
	\mathbf{t}^{+} = \mathbf{t} - \mathcal{F}(\mathbf{t}; \mathcal{W}),\label{eq:model}
\end{equation}
where $\mathcal{F}$ is a residual function with parameters
$\mathcal{W}$ and $\mathbf{t}^{+}$ is the output of this module with identity skip
connections.
Using this basic model, our goal reduces to learn a mapping $\mathcal{F}$ to formulate
the propagation residual of transmissions.
We can see that different from
strictly sequential networks, the module in \eqref{eq:model} introduces identity skip-connections that bypass residual layers, allowing transmissions to flow
from any layers directly to any subsequent layers.

As for the specific form of $\mathcal{F}$, we define it by aggregating a data driven submodule
(denoted as $\mathcal{D}$) and a prior driven submodule (denoted as $\mathcal{P}$), i.e.,
\begin{equation}
	\mathcal{F}(\mathbf{t}; \mathcal{W}) = \mathcal{D}(\mathbf{t}; \mathcal{W}_{D})+\lambda_{P}\mathcal{P}(\mathbf{I}).\label{eq:residual}
\end{equation}
Here $\mathcal{W}=\{\mathcal{W}_D,\lambda_{P}\}$ are learnable parameters in which $\mathcal{W}_D$ denote
the propagation parameters of $\mathcal{D}$ and $\lambda_{P}\geq 0$ is a weight parameter to penalize $\mathcal{P}$.
In the following, we deduce particular formulations of these architectures based on training data and physical principles, respectively.

\textbf{Data Submodule:}
We first design the data driven submodule $\mathcal{D}$ as a cascade of two convolutions and one activation to fit transmission propagation:
\begin{equation}
	\mathcal{D}(\mathbf{t};\mathcal{W}_D) = \sum\limits_{k=1}^K\hat{\bm{\omega}}_k\otimes\phi_k(\check{\bm{\omega}}_k\otimes\mathbf{t}),\label{eq:data}
\end{equation}
where $\{\phi_k\}_{k=1}^K$ denote nonlinear activations, $\otimes$ denotes convolution operator and $\{\hat{\bm{\omega}}_k, \check{\bm{\omega}}_k\}_{k=1}^K$
are pairs of convolutional filters before and after each activation.
It is clear that propagating $\mathbf{t}$ only with $\mathcal{D}$ will definitely discard rich prior information of hazy images.
Thus it is necessary to further utilize our domain knowledge (e.g., Eq.~\eqref{eq:haze}) to control this
transmission propagation toward desired stable state.

\textbf{Prior Submodule:}
We now discuss how to design our prior submodule
$\mathcal{P}:\mathbf{I}(x)\to t(x)$ to incorporate
visual cues to guide the data driven transmission propagation.
Specifically, by reformulating Eq.~\eqref{eq:haze}, we have that $t(x)$ is actually the ratio of two line segments~\cite{he2011single} at each pixel location, i.e.,
\begin{equation}
	t(x) =\frac{\|\mathbf{I}_{\mathbf{A}}(x)\|}{\|\mathbf{J}(x)-\mathbf{A}\|}=\frac{{\mathbf{I}_{\mathbf{A}}}_c(x)}
	{\mathbf{J}_c(x)-\mathbf{A}_c},  \label{eq:t0}
\end{equation}
where $c\in\{r, g, b\}$ is the color channel index, ${\mathbf{I}_{\mathbf{A}}}_c(x)=\mathbf{I}_c(x)-\mathbf{A}_c$ is the translated observation such that the airlight is at the origin.
Since the pixel values of the latent scene radiance $\mathbf{J}(x)$ should be bounded, it is natural to enforce the following constraints on $\mathbf{J}(x)$:
$\check{\alpha}\check{\mathbf{I}}_c\leq\mathbf{J}_c\leq\hat{\alpha}\hat{\mathbf{I}}_c$, where $\hat{\alpha},\check{\alpha}\geq 0$ are two scaling parameters and $\hat{\mathbf{I}}_c$ and $\check{\mathbf{I}}_c$ denote the maximum and minimum value of $\mathbf{I}_c$, respectively.
Then by combining this inequality with Eq.~\eqref{eq:t0}, we can obtain the prior module as
\begin{equation}
	\mathcal{P}(\mathbf{I})=\mathcal{T}_{[0,1]}\left(\max\limits_{c}\left(\frac{{\mathbf{I}_{\mathbf{A}}}_c(x)}{\hat{\alpha}\hat{\mathbf{I}}_c-\mathbf{A}_c},
	\frac{{\mathbf{I}_{\mathbf{A}}}_c(x)}{\check{\alpha}\check{\mathbf{I}}_c-\mathbf{A}_c}\right)\right),\label{eq:prior}
\end{equation}
where $\mathcal{T}_{[0,1]}$ denotes the projection on the range $[0,1]$.
 By setting $\check{\alpha}=0$ in Eq.~\eqref{eq:prior}, we can observe that DCP formulation is just a special case of $\mathcal{P}$. However,
it is known that the original DCP always tends to provide inexact transmission estimation on specific regions (e.g., bright sky and headlights of cars). Fortunately, it will be demonstrated in the experimental part that our proposed prior module in Eq.~\eqref{eq:prior} actually provides an efferent manner to avoid incorrect transmission estimations on these challenging hazy images.


We illustrate the overall propagation pipeline of DPATN in Figure~\ref{fig:short}. The dehazing results of DPATN with different architectures (i.e.,
$\mathcal{P}$, $\mathcal{D}$ and $\mathcal{F}$) are further demonstrated in Figure~\ref{fig:component}. In this experiment, we also plot results generated by the naive combination of the prior and the media filter (denoted as ``P + MF'') as our baseline. It can be seen that in the sky region, the results of the prior $\mathcal{P}$ tends to estimate excessive depth and thus leads to obvious artifacts.
While the purely data-based $\mathcal{D}$ may also result in unclear restoration (e.g., with low contrast near the building boundary). We can see that the media filter only slightly improves the performance of $\mathcal{P}$. But it cannot correct the improper scene depth, thus there still exists severe artifacts in its result. In contrast, our proposed aggregation architecture $\mathcal{F}$ can provide more accurate transmission map and much better dehazing result. Moreover, the aggregated $\mathcal{F}$ obtains the highest quantitative score (i.e., PSNR) among all the compared strategies.
\begin{figure*}[t]
	\begin{center}
		\begin{tabular}{c@{\extracolsep{0.3em}}c@{\extracolsep{0.3em}}c@{\extracolsep{0.3em}}c@{\extracolsep{0.3em}}c@{\extracolsep{0.3em}}c}

			\includegraphics[width=.152\textwidth]{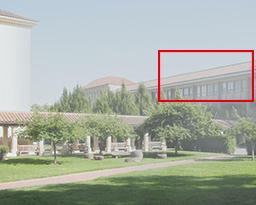}
			&\includegraphics[width=.152\textwidth]{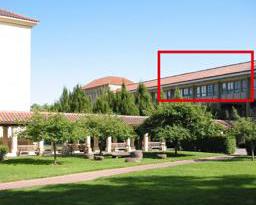}
			&\includegraphics[width=.152\textwidth]{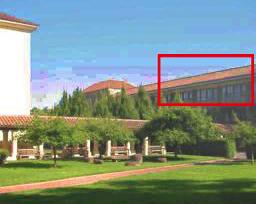}
			&\includegraphics[width=.152\textwidth]{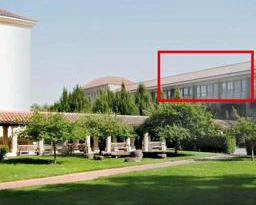}
			&\includegraphics[width=.152\textwidth]{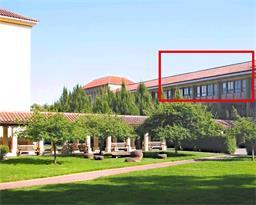}
			&\includegraphics[width=.152\textwidth]{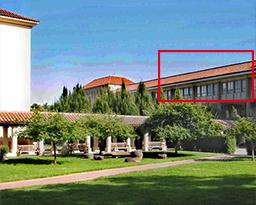}
			\\
			\includegraphics[width=.152\textwidth]{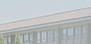}
			&\includegraphics[width=.152\textwidth]{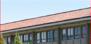}
			&\includegraphics[width=.152\textwidth]{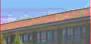}
			&\includegraphics[width=.152\textwidth]{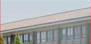}
			&\includegraphics[width=.152\textwidth]{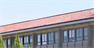}
			&\includegraphics[width=.152\textwidth]{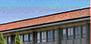}
			\\
			
			&\includegraphics[width=.152\textwidth]{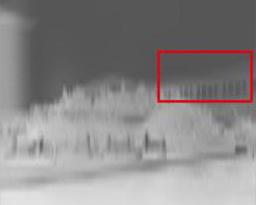}
			&\includegraphics[width=.152\textwidth]{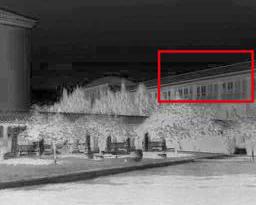}
			&\includegraphics[width=.152\textwidth]{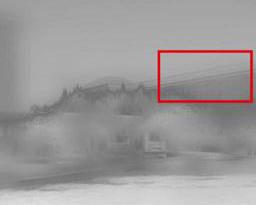}
			&\includegraphics[width=.152\textwidth]{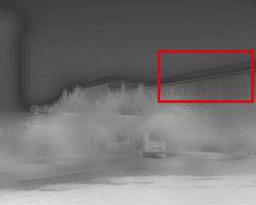}
			&\includegraphics[width=.152\textwidth]{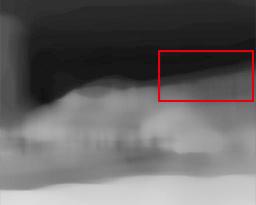}
			\\Input  & Ground Truth & Prior (23.25) & Data (23.13) & Agg  (\textbf{26.82}) &Prior + MF (24.56)
		\end{tabular}
	\end{center}
	\caption{The dehazing results obtained by DPATN with different components: ``Prior'', ``Data'', and their aggregation (``Agg'' for short). Top two rows: the haze removal results with zoomed-in comparisons. Bottom row: the estimated transmission maps. On the right most column, we also illustrate the results of the ``Prior + MF'' strategy. Here ``MF'' denotes the median filtering operation. The quantitative scores (i.e., PSNR) are also reported in the brackets. }\label{fig:component}
\end{figure*}

\subsection{Lightweight Learning Framework}\label{sec:learn}

As our goal is to learn an aggregated residual network to map the hazy image to its corresponding transmission, we first collect a set of training pairs
$\{\mathbf{I}_s, \mathbf{t}_s^*\}_{s=1}^S$, where $\mathbf{I}_s$ is a hazy image and $\mathbf{t}_s^*$ is the corresponding
ground-truth transmission map.
Then we define a quadratic training cost function to measure the difference between the output of DPATN
and the ground-truth transmission map as
\begin{equation}
\mathcal{J}(\mathbf{t}^L, \mathbf{t}^*;\{\mathcal{W}^{l}\}_{l=1}^L)=\frac{1}{2}\|\mathbf{t}^L-\mathbf{t}^*\|_2^2.
\end{equation}
Therefore, the learning process of DPATN can be formulated
as a deep network constrained optimization model:
\begin{equation}
\left\{
\begin{array}{l}
\min\limits_{\{\mathcal{W}^{l}\}_{l=1}^L}\sum\limits_{s=1}^S\mathcal{J}(\mathbf{t}_s^L, \mathbf{t}_s^*;\{\mathcal{W}^{l}\}_{l=1}^L),\\
s.t. \ \mathbf{t}_s^L = \mathbf{t}_s^0 + \sum\limits_{l=0}^{L-1}\mathcal{F}(\mathbf{t}_s^{l}; \mathcal{W}^{l}).
\end{array}
\right.\label{eq:train}
\end{equation}

The key is to compute the gradients of the training cost $\mathcal{J}$
with respect to the parameters $\mathcal{W}^l$. By the chain rule of backpropagation, we have
\begin{equation}
\begin{array}{r}
\frac{\partial \mathcal{J}(\mathbf{t}^L, \mathbf{t}^*)}{\partial \mathcal{W}^l}=
\frac{\partial \mathcal{J}(\mathbf{t}^L, \mathbf{t}^*)}{\partial \mathbf{t}^l}\frac{\partial \mathbf{t}^l}{\partial \mathcal{W}^l}=\frac{\partial \mathcal{J}(\mathbf{t}^L, \mathbf{t}^*)}{\partial \mathbf{t}^L}
\frac{\partial \mathbf{t}^{L}}{\partial \mathbf{t}^{l}}\frac{\partial \mathbf{t}^l}{\partial \mathcal{W}^l}\\
=\frac{\partial \mathbf{t}^l}{\partial \mathcal{W}^l}\frac{\partial \mathcal{J}(\mathbf{t}^L, \mathbf{t}^*)}
{\partial \mathbf{t}^L}\left(1+ \frac{\partial}{\partial \mathbf{t}^l}\sum\limits_{i=l}^{L-1}\mathcal{D}(\mathbf{t}^i,\mathcal{W}^i)\right).\\
\end{array}\label{eq:gradient}
\end{equation}

In general, one may follow standard training schemes (e.g.,~\cite{he2015deep})
to learn all the network parameters from training data using Eq.~\eqref{eq:gradient}. However, such direct strategy may lead to extremely high computational burden.
Therefore, we
adopt a lightweight learning framework for DPATN, instead. That is, we parameterize network architectures based on their specific structures to reduce the network parameters as follows:
We represent each filter as
$\bm{\omega}_k=\sum_{i}\alpha_{k,i}\mathbf{d}_i$
where $\{\mathbf{d}_i\}$ is the DCT basis and $\{\alpha_{k,i}\}$ is the set of filter coefficients to be learned.
As for each transformation $\phi_k$, we parameterize it
using piecewise linear functions determined by a set of control points $\{p_i,q_{k,i}\}$, i.e., $\phi_k=\sum_i\beta_{k,i}q_{k,i}$
where $\{p_i\}$ are predefined positions uniformly located within $[-1, 1]$,
$\{q_{k,i}\}$ are the function values at these positions and $\{\beta_{k,i}\}$ are combination coefficients.
In this way, the number of parameters in DPATN can be significantly reduced. Besides, the derivability
of $\phi_k$ required by analysis in Section~\ref{sec:ebm} can also be guaranteed by such parameterization.
Finally, we follow suggestions in \cite{chen2015learning} to utilize L-BFGS algorithm to optimize Eq.~\eqref{eq:train}.

\subsection{Latent Scene Radiance Recovery}\label{sec:dehazing}

Now we are ready to discuss how to recover scene radiance $\mathbf{J}$ from hazy observation $\mathbf{I}$.
In addition to the medium transmission $\mathbf{t}$,
we need to estimate the global atmospheric light $\mathbf{A}$ in Eq.~\eqref{eq:haze}.

 The most commonly used strategy~\cite{he2011single,ren2016single} is to select several brightest pixels (e.g., $0.1\%$) in the dark channel. Among these pixels, the one with the highest intensity in the image is selected as the atmospheric light.
Here
we improve this strategy based on a filtering process~\cite{meng2013efficient}, i.e.,
$\mathbf{A}^c=\max\left(\mathtt{minfilter}(\mathbf{I}^c)\right)$,
where the ``$\max$'' operation outputs the maximum values and ``$\mathtt{minfilter}$'' denotes the sliding window minimum filter.
With the estimated $\mathbf{t}$ and $\mathbf{A}$, we can recover the haze-free image using Eq.~\eqref{eq:haze}.
To avoid division by zeros, $\mathbf{J}$ is particularly calculated by
\begin{equation}
	\mathbf{J}(x)=\mathbf{A}+(\mathbf{I}(x)-\mathbf{A})/(\max(t(x),\epsilon)),\label{eq:j}
\end{equation}
where $\epsilon\geq 0$ is a small constant.

\section{Energy-Based Propagation Investigation}\label{sec:ebm}

The concept of ``energy'' in LeCun's energy-based-model (EBM)~\cite{lecun2006tutorial} is referring to a parameterized function that maps points in input space to a scalar such that the desired points get assigned low energies, while the incorrect ones are assigned high energies. It has been demonstrated that various learning frameworks (e.g.,
probabilistic models~\cite{marc2007unified}, partial differential equation models~\cite{liu2016learning} and deep networks~\cite{zhao2016energy,liu2018proximal,liu2018toward}) can be reformulated in the light of ``energy minimization''.

In this section, we would like to utilize EMB perspective to investigate the propagation behavior of our aggregated transmission network. That is, we aim to deduce a variational formulation from the developed DPATN, thus provide a more straightforward way to understand the insights in DPATN. Indeed, different from standard methodology, which deduces the practical calculation scheme from the designed variational energy, here we establish an underlying energy model from our DPATN propagation scheme, inversely.

Firstly, based on the recursive architecture in \eqref{eq:model}, we can explicitly build connections between the input $\mathbf{t}^0$ and output $\mathbf{t}^L$
of the network (with $L$ residual modules) as follows:
\begin{equation}
\mathbf{t}^L = \left(\mathbf{t}^0+\sum_{l=0}^{L-1}\lambda_{P}^l\mathcal{P}(\mathbf{I})\right) + \sum\limits_{l=0}^{L-1}\mathcal{D}(\mathbf{t}^{l}; \mathcal{W}_D^{l}),
\end{equation}
where $\{\mathbf{t}^l\}_{l=1}^{L-1}$ are intermediate transmissions.
Suppose there exist the derivable functions
$\rho_k$ satisfying\footnote{As presented in Section~\ref{sec:learn},
	this property can be guaranteed by a specifically designed parameterization scheme.}
\begin{equation}
\partial \rho_k /\partial \mathbf{t}=-\phi_k,
\end{equation}
for each $\phi_k$, then we can obtain an energy
\begin{equation}
\mathcal{L}_D=\sum_{k=1}^K\rho_k(\hat{\bm{\omega}}_k\otimes\mathbf{t}).\label{eq:energy-l}
\end{equation}
Thus it is easy to understand that the data submodule $\mathcal{D}$ is just the negative gradient direction of $\mathcal{L}_D$, i.e.,
\begin{equation}
\mathcal{D}(\mathbf{t};\mathcal{W}_D)=-\partial \mathcal{L}_D(\mathbf{t};\mathcal{W}_D)/\partial \mathbf{t},
\end{equation}
in which we implicitly enforce that $\check{\bm{\omega}}_k$ is obtained by rotating $\hat{\bm{\omega}}_k$ $180$ degrees.
 So in summary, DPATN actually builds a \emph{learnable
	propagation (i.e., gradient descent trajectory)} with initial status $\tilde{\mathbf{t}}^0=\mathbf{t}^0+\sum_{l=0}^{L-1}\lambda_{P}^l\mathcal{P}(\mathbf{I})$
to minimize the following transmission energy:
\begin{equation}
\min\limits_{\mathbf{t}^L}\mathcal{L}_D(\mathbf{t}^L;\{\mathcal{W}_D^l\}_{l=0}^{L-1}),\label{eq:engery}
\end{equation}
where we explicitly consider $\{\mathcal{W}_D^l\}_{l=0}^{L-1}$ as parameters of $\mathbf{t}^L$ to emphasize the recursive nature of this energy.

By investigating the derived form of $\mathcal{L}_D$ in Eq.~\eqref{eq:energy-l}, we can link the energy minimization model in Eq.~\eqref{eq:engery} to the maximum likelihood estimation problem with the following probability distribution
\begin{equation}
p(\mathbf{t}^l)\propto\prod\limits_{k=1}^K
\exp(-\rho_k^l(\hat{\bm{\omega}}_k^l\otimes\mathbf{t}^l)), \ \mbox{for any} \ 1\leq l \leq L.\label{eq:gibbs}
\end{equation}
Therefore, we can also interpret DPATN from Bayesian viewpoint, i.e., the behavior of DPATN can be understood as a prior guided
prorogation with a learnable Gibbs distribution~\cite{schmidt2016cascades} on each intermediate and output transmissions.

\emph{In summary, DPATN provides a new way to aggregate the describing power obtained from data dependent deep architectures
	and visual priors revealed by the physical rule, thus bridges the gap between domain knowledge and training data in dehazing task.}

\textbf{Remarks:} As a nontrivial byproduct, the above
analysis may also be helpful for investigating insights and deducing new architectures for other vision tasks.
For example, residual network (ResNet)~\cite{he2015deep} has achieved great success in high-level recognition/classification tasks. Though some efforts from the experimental point of view
have been made~\cite{he2016identity,veit2016residual}, it is still
difficult to integrate the intrinsic mechanism of ResNet.
We want to point out that the Gibbs distribution in~\eqref{eq:gibbs} can be directly utilized to interpret
the propagation behaviors of ResNet.
This is because the data driven submodule in DPATN
actually share the same structure with the basic unit in ResNet. Furthermore, following our methodology,
it is also possible to incorporate additional human perspectives to improve
the recognition/classification performance of standard ResNet.

\section{Extensions with Task-Aware Image Separation} \label{extension}

To formulate more challenging image enhancement problems,
we would like to extend the proposed propagation network using a generalized atmospheric scattering model, i.e.,
\begin{equation}
\mathbf{I}(x)=\mathbf{J}(x)t(x) + \mathbf{A}(1-t(x))+\mathbf{P}(x),\label{eq:ge-haze}
\end{equation}
where $\mathbf{P}(x)$ is an additional problem-dependent image layer.
It is observed that we actually combine three different terms, i.e., the direct transmission $\mathbf{J}(x)t(x)$, airlight $\mathbf{A}(1-t(x))$ and
task-related layer $\mathbf{P}(x)$ and assume that the incoming light intensity to a camera is linearly proportional to
the camera's pixel values. Similar to Eq.~\eqref{eq:haze}, within this extended framework, the transmission map $t$ also plays the main role in eliminating
the hazing effect in the image. Meanwhile, the additional term $\mathbf{P}$ will be used to incorporate image priors into
the process.

\subsection{Task-Aware Image Layer Separation}

In this part, we develop a new image separation formulation together with a flexible plug-and-play optimization scheme to extend DPATN for more challenging image enhancement tasks, such as underwater image enhancement and singe image rain removal. We first recognize that it is necessary to explicitly formulate the unknown corruptions (e.g. color-shift or rain layer) in these problems. Therefore, we consider $\mathbf{P}$ as the corruption layer and introduce a latent observation variable $\mathbf{L}(x)=\mathbf{J}(x)t(x) + \mathbf{A}(1-t(x))$ in Eq.~\eqref{eq:ge-haze}. In this way, these problems can be addressed by first estimating the latent image layer $\mathbf{L}$ from corrupted observation $\mathbf{I}$ and then calculating intrinsic transmission $t$ from $\mathbf{L}$.

Specifically, we consider the following task-aware image layer separation problem (to remove the effects of $\mathbf{P}$):
\begin{equation}
\quad\min\limits_{0\leq\mathbf{L},\mathbf{P}\leq\mathbf{I}}\frac{1}{2}\|\mathbf{I}-\mathbf{L}-\mathbf{P}\|^2 +\delta(\mathbf{L}) + \sigma(\mathbf{P}),\label{eq:task-aware}
\end{equation}
where $\delta$ and $\sigma$ are the regularization terms on $\mathbf{L}$ and $\mathbf{P}$, respectively.
Rather than directly solving Eq.~\eqref{eq:task-aware} using standard solvers, this work provides a flexible plug-and-play scheme to incorporate different strategies to optimize the general image layer separation model for particular problems.
Specifically, we first introduce a half-quadratic formulation~\cite{geman1995nonlinear} with two auxiliary variables $\mathbf{Y}_L$ and $\mathbf{Y}_P$ to Eq.~\eqref{eq:task-aware}:
\begin{equation}
\min\limits_{\mathbf{L}, \mathbf{P}, \mathbf{Y}_L, \mathbf{Y}_P}\frac{1}{2}\|\mathbf{I}-\mathbf{L}-\mathbf{P}\|^2 +\tilde{\delta}(\mathbf{L},\mathbf{Y}_L) + \tilde{\sigma}(\mathbf{P},\mathbf{Y}_{P}).\label{eq:task-aware-hq}
\end{equation}
The half-quadratic penalized priors $\tilde{\delta}$ and $\tilde{\sigma}$ are defined as
\begin{equation}
\begin{array}{l}
\tilde{\delta}(\mathbf{L};\mathbf{Y}_L):=\min\limits_{0\leq\mathbf{Y}_L\leq\mathbf{I}}\delta(\mathbf{Y}_L)+ \frac{\mu_L}{2}\|\mathbf{Y}_L-\mathbf{L}\|^2,\\
\tilde{\sigma}(\mathbf{P};\mathbf{Y}_P):=\min\limits_{0\leq\mathbf{Y}_P\leq\mathbf{I}}\sigma(\mathbf{Y}_P)+ \frac{\mu_P}{2}\|\mathbf{Y}_P-\mathbf{P}\|^2,
\end{array}\label{eq:hq-prior}
\end{equation}
where $\mu_L,\mu_P>0$ are penalty parameters. Then our task-aware image separation can be summarized as
\begin{equation}
\left\{\begin{array}{l}
\mathbf{L}^+ = \frac{1}{1+\mu_L}(\mathbf{I}-\mathbf{P} +\mu_L\mathbf{Y}_L^+), \ \mbox{with} \ \mathbf{Y}_L^+ = \mathcal{A}_{\mathbf{L}}(\mathbf{L}),\\
\mathbf{P}^+ = \frac{1}{1+\mu_P}(\mathbf{I}-\mathbf{L}^+ +\mu_P\mathbf{Y}_P^+), \ \mbox{with} \ \mathbf{Y}_P^+ = \mathcal{A}_{\mathbf{P}}(\mathbf{P}),\\
(\mu_L^+,\mu_P^+)=\eta(\mu_L,\mu_P), \ \mbox{with} \ \eta>1.
\end{array}\right.\label{eq:pp-hq}
\end{equation}
We will demonstrate in the following section that the operators $(\mathcal{A}_{\mathbf{L}}, \mathcal{A}_{\mathbf{P}})$ are related to specific tasks. Notice that we also composite a normalization process provided in \cite{li2014single} to these two operations
to fit the general bound constraints in Eq.~\eqref{eq:task-aware}. To end this part, we provide a theorem to claim that the proposed iterations will always converge to a fixed-point.
\begin{thm}(Fixed-point Convergent)\label{thm:convergence}
	The proposed propagations of task-aware image layers separation are globally convergent\footnote{Here ``globally convergent'' indicates that the whole sequence (but not any subsequences) generated by Eq.~\eqref{eq:pp-hq} is converged. Notice that this concept is fundamental and has been widely used in non-convex optimization society (see~\cite{bolte2014proximal} for example)}. That is, the sequence $\{(\mathbf{L}^k,\mathbf{P}^k,\mathbf{Y}_L^k,\mathbf{Y}_P^k)\}$ generated by Eq.~\eqref{eq:pp-hq} is a Cauchy sequence, thus converge globally to a fixed-point.
\end{thm}
\begin{proof}
	It is easy to check that the constraint $0\leq\mathbf{L},\mathbf{P}\leq\mathbf{I}$ 
	actually enforces a bound assumption to the operations $\mathcal{A}_{\delta}$ and $\mathcal{A}_{\sigma}$.
	So we can have the following two inequalities 
	\begin{equation}
	\|\mathcal{A}_{\delta}(\mathbf{L})-\mathbf{L}\|^2\leq \frac{C_L}{\mu_L} \ \mbox{and} \ \|\mathcal{A}_{\sigma}(\mathbf{P})-\mathbf{P}\|^2\leq \frac{C_P}{\mu_P},
	\end{equation}
	where $0<C_L, C_P<\infty$ are two constants.
	By optimal conditions of Eq.~\eqref{eq:task-aware-hq}, we further have that
	\begin{equation}
	\begin{array}{l}
	\|\mathbf{L}^{k+1}-\mathbf{Y}_L^{k+1}\| =\frac{1}{\mu_L^k}\|\nabla_{\mathbf{L}} f(\mathbf{L},\mathbf{P}^k)\|\leq \frac{C}{\mu_L^k},\\
	\|\mathbf{P}^{k+1}-\mathbf{Y}_P^{k+1}\| =\frac{1}{\mu_P^k}\|\nabla_{\mathbf{P}} f(\mathbf{L}^{k+1},\mathbf{P})\|\leq \frac{C}{\mu_P^k},
	\end{array}
	\end{equation}
	where $f(\mathbf{L},\mathbf{P})=\frac{1}{2}\|\mathbf{I}-\mathbf{L}-\mathbf{P}\|^2$ and $0<C<\infty$ is a constant.
	So it can be derived that 
	\begin{equation}
	\begin{array}{l}
	\quad\|\mathbf{L}^{k+1}-\mathbf{L}^k\| \leq \|\mathbf{L}^{k+1}-\mathbf{Y}_L^{k+1}\| + \|\mathbf{Y}_L^{k+1}-\mathbf{L}^k\|\\
	\leq  \frac{C}{\mu_L^{k}} + \|\mathcal{A}_{\delta}(\mathbf{L}^k)-\mathbf{L}^k\| \leq \frac{C+C_L}{\mu_L^k}.
	\end{array}\label{eq:L}
	\end{equation}
	\begin{equation}
	\begin{array}{l}
	\quad\|\mathbf{Y}_L^{k+1}-\mathbf{Y}_L^k\|=\|\mathcal{A}_{\delta}(\mathbf{L}^k)-\mathcal{A}_{\delta}(\mathbf{L}^{k-1})\|\\
	\leq\sum\limits_{j=k-1}^{k}\|\mathcal{A}_{\delta}(\mathbf{L}^j)- \mathbf{L}^j\| +\| \mathbf{L}^k- \mathbf{L}^{k-1}\| \\
	\leq \frac{C}{\mu_L^k} + \frac{C+C_L}{\mu_L^{k-1}}+\frac{C}{\mu_L^{k-1}}=\frac{(2\beta+1)C+\beta C_L}{\mu_L^{k}}.
	\end{array}\label{eq:Y}
	\end{equation}
	Using the same methodology as that in Eq.~\eqref{eq:L}-\eqref{eq:Y}, we also have the finite length property for $\{\mathbf{P}^k\}$ and $\{\mathbf{Y}_L^k\}$, i.e.,
	\begin{equation}
	\begin{array}{l}
	\|\mathbf{P}^{k+1}-\mathbf{P}^k\| \leq\frac{C+C_P}{\mu_P^k},\\
	\|\mathbf{Y}_L^{k+1}-\mathbf{Y}_L^k\|\leq\frac{(2\beta+1)C+\beta C_P}{\mu_P^{k}}.
	\end{array}  
	\end{equation}
	Thus we have that $\|\mathbf{L}^{k+1}-\mathbf{L}^k\|\to 0$, $\|\mathbf{Y}_L^{k+1}-\mathbf{Y}_L^k\|\to 0$,
	$\|\mathbf{P}^{k+1}-\mathbf{P}^k\|\to 0$ and $\|\mathbf{Y}_P^{k+1}-\mathbf{Y}_P^k\|\to 0$
	as $k\to\infty$. So  $\{(\mathbf{L}^k,\mathbf{P}^k,\mathbf{Y}_L^k,\mathbf{Y}_P^k)\}$ is a Cauchy sequence, thus is globally converged to a fixed-point, which concludes our proof. 
\end{proof}

In the following, we demonstrate how to apply the proposed image separation based extension of DPATN for more challenging image enhancement tasks, such as underwater image enhancement and single image rain removal.

\subsection{Underwater Image Enhancement}

%

As shown in~\cite{Chiang2012restore}, we can directly use background light to approximate the true in-scattering term in the full radiative transport equation and thus achieve the following underwater imaging model
\begin{equation}
\left\{\begin{array}{l}
\mathbf{I}_c(x)=\mathbf{J}_c(x)t_c(x) + \mathbf{B}_c(1-t_c(x)),\\
t_c(x)=\exp(-\beta_c d(x)),
\end{array}\right.\label{eq:underwater}
\end{equation}
where $\mathbf{B}=\{\mathbf{B}_{c}\}_{c\in\{r,g,b\}}$ is the homogeneous background light of different channels. Notice that all the variables in Eq.~\eqref{eq:underwater} are suffering from the effects of both light scattering and color changes by light with wavelength $c\in\{r,g,b\}$ and thus $t_c$ should be considered as a function of both wavelength and the object-camera distance. We observe that the above underwater imaging model
shares similar formulation with the atmospheric scattering model in Eq.~\eqref{eq:haze} at each wavelength domain. So it is possible to apply DPATN to enhance the quality of underwater degraded images.

It is known that both light scattering and color change are sources of distortion for underwater photography.
So different from standard dehazing task, it is necessary to incorporate image layer separation process in Eq.~\eqref{eq:task-aware} to address color shift on underwater images. Inspired by the color constancy assumption in \cite{li2014single},
the following cross-channel constraint is incorporated into Eq.~\eqref{eq:task-aware} to balance the range of intensity values in RGB channels
\begin{equation}
	\sum\limits_{x}\bar{\mathbf{J}}_r(x)=
	\sum\limits_{x}\bar{\mathbf{J}}_g(x)=\sum\limits_{x}\bar{\mathbf{J}}_b(x).
\end{equation}
As for prior functions, we follow \cite{li2014single}
to specify $\delta(\mathbf{L})=\min(\nabla\mathbf{L}(x)^2,\tau)$ to preserve large gradients
while $\sigma(\mathbf{P})=\|\Delta\mathbf{P}\|^2$ to smooth the color shift. Then the transmission propagation in Eq.~\eqref{eq:model} are performed
based on $\mathbf{L}$ to enhance the given underwater images.

As red light attenuates faster than blue counterpart in underwater propagation, the degraded images often show bluish tone rather than white tone in hazy images. To avoid interference from white objects and get a preciser $\mathbf{B}_c$ in initial prior submodule, we build a subset with the brightest 0.1\% pixels of the image (denoted as $\Omega$) and calculate $\mathbf{B}$
as follows
\begin{equation}
\mathbf{B}=\mathbf{I}\left(\arg\max\limits_{x\in\Omega}\left(\mathbf{I}^c(x)-\mathbf{I}^r(x)\right)\right), \quad c\in\{g, b\}.
\end{equation}

\subsection{Single Image Rain Removal}

Single image rain removal is a challenging task, as raindrops always
obstruct background scenes, resulting in several types of visibility degradations. Most existing models often directly decompose images as a rain layer and a background layer. However, due to the intrinsic overlapping between rain streaks and background texture patterns, such simple separation model is insufficient to cover some important factors in real rain images. For example, it is observed that distant rain streaks indeed accumulate and generate atmospheric veiling effects similar to mist or fog, which severely reduce the visibility by scattering light out and into the line of sight.

Therefore, it is more reasonable to utilize our generalized atmospheric scattering model Eq.~\eqref{eq:ge-haze} to formulate the rainy observation $\mathbf{I}$, in which $t$ is still the transmission, while $\mathbf{P}$ and $\mathbf{J}$
denote the rain streaks and background scene, respectively.
Then the deraining task can be considered as removing the rain layer $\mathbf{P}$
from the observation $\mathbf{I}$, estimating the transmission $t$ and then recovering the latent background $\mathbf{J}$ based on Eq.~\eqref{eq:ge-haze}. So we can address this task by extending DPATN with layer separation.
As for the prior on $\mathbf{P}$, we just adopt a Gaussian mixture model (GMM) \cite{li2016rain} to investigate the implicit distributions of the rain streaks. Specifically, we follow~\cite{li2016rain} to define $\sigma(\mathbf{P})=\sum_{x\in\Omega}\log\mathcal{G}(\mathcal{H}(\mathbf{P}(x)))$,
where $\mathcal{H}(\mathbf{P}(x))$ is to extract the patch around $\mathbf{P}(x)$ (also remove its mean) and $\mathcal{G}(\cdot)$ stands for the GMM model.

\section{Experimental Results}\label{sec:exp}

\begin{figure*}[!htbp]
	\begin{center}
			\begin{tabular}{c@{\extracolsep{0.3em}}c@{\extracolsep{0.3em}}c@{\extracolsep{0.3em}}c@{\extracolsep{0.3em}}c}
	\multirow{2}{*}{\includegraphics[width=.19\textwidth]{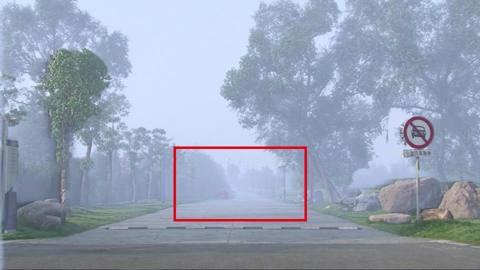}}
	&\includegraphics[width=.19\textwidth]{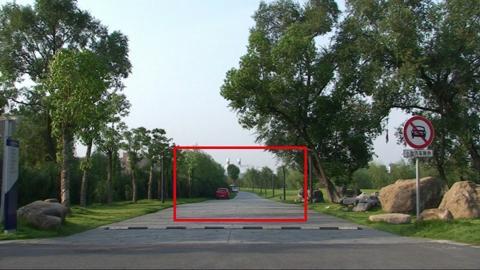}
	&\includegraphics[width=.19\textwidth]{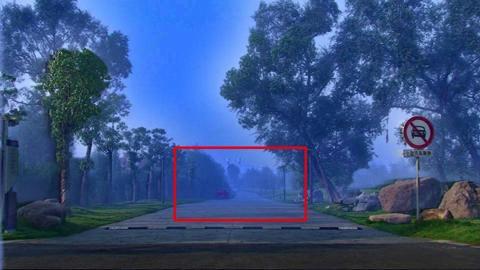}
	&\includegraphics[width=.19\textwidth]{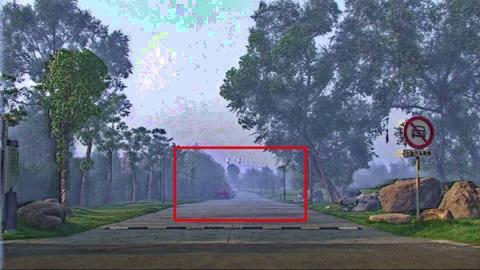}
	&\includegraphics[width=.19\textwidth]{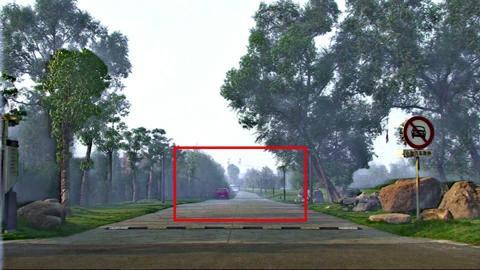}\\
	&\includegraphics[width=.19\textwidth]{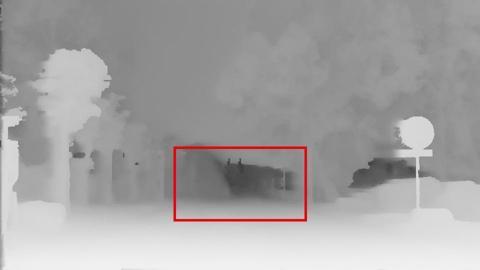}
	&\includegraphics[width=.19\textwidth]{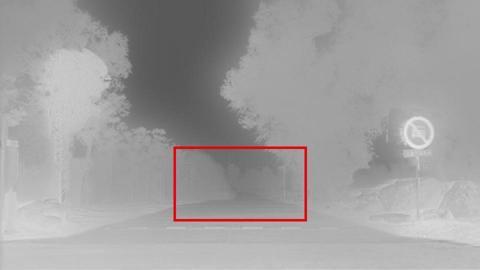}
	&\includegraphics[width=.19\textwidth]{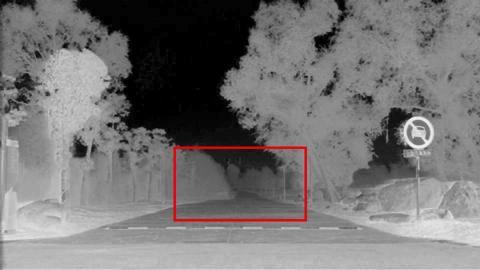}
	&\includegraphics[width=.19\textwidth]{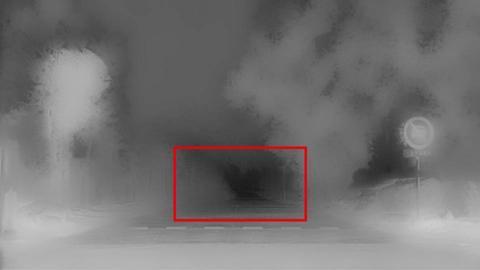}\\
	Input & Ground Truth  &  He~\emph{et al.} & Meng~\emph{et al.}  & Ours
\end{tabular}
	\end{center}
	\caption{Comparison with He~\emph{et al.} and Meng~\emph{et al.} (i.e., physical priors) on synthetic hazy image. The dehazing results and estimated transmission maps are presented on the top and bottom rows, respectively.}\label{fig:physical}
\end{figure*}

We first consider the task of single image haze removal and evaluate the proposed algorithm (i.e., DPATN) on both synthetic datasets~\cite{silberman2012indoor,fattal2014dehazing} and real-world hazy images ~\cite{he2011single,fattal2014dehazing} with comparisons to \emph{eight} recent
state-of-the-art methods, including prior driven approaches (e.g., \cite{he2011single,zhu2015fast,meng2013efficient,berman2016non,li2016haze,chen2016robust}) and fully data driven CNNs (e.g., \cite{cai2016dehazenet,ren2016single}). We use either
their original implementations or the results provided by the authors for fair comparison.

As for DPATN, we cascade 5 residual modules, in which consist of 24 convolutional filters with the size $5\times 5$
and 24 nonlinear activation functions. The filters and activations are respectively initialized by the
average-discarded DCT basis
and the unified influence function $2z/(1+z^2)$. We train DPATN on a synthesized hazy image set, in which observations are simulated by clean images and known depth maps all from NYU depth database~\cite{silberman2012indoor}.
We generate atmospheric light $\mathbf{A} = [a;a;a]$ ($a \in [0.7,1.0]$) and
medium extinction coefficient $\beta \in [0.7,1.2]$
for each clear image. Then the corresponding synthesized hazy images are generated by Eq.~\eqref{eq:haze}.
Finally, we crop a $180 \times 180$ region from each transmission, resulting in our training set with images all of the size $180 \times 180$.
The parameters of DPATN are experimentally set as follows:
In the phase of latent scene radiance recovery, we set $\check{\alpha} \in [1.5,6 ]$, $\hat{\alpha}=1.5$ in Eq.~\eqref{eq:prior}, and $\epsilon=0.01$ in Eq.~\eqref{eq:j}. As for image layer separation formulation, we initialize the penalty parameters $ \mu_L=0.1$ and $\mu_P=0.5$, and then update them with the ratio $\eta=1.05$ at each iteration.


To perform quantitative comparisons, both Peak Signal-to-Noise Ratio (PSNR) and Structural Similarity (SSIM) metrics
are used to measure the similarity between the restored results and the ground truth (GT).
We also report one non-reference metric, named Natural Image Quality Evaluator (NIQE) \cite{liu2014no} for real-world hazy images (without GTs). Notice that the lower NIQE score indicates the better performance. All the experiences are conducted on a PC with Intel Core i7-3770 CPU at 3.4 GHz, 8 GB RAM and a NVIDIA GeForece GTX 1050 Ti GPU.

	\begin{figure}[!htbp]
	\begin{center}
		\begin{tabular}{c@{\extracolsep{0.3em}}c@{\extracolsep{0.3em}}c@{\extracolsep{0.3em}}
				c@{\extracolsep{0.3em}}c}
			\includegraphics[width=.24\textwidth]{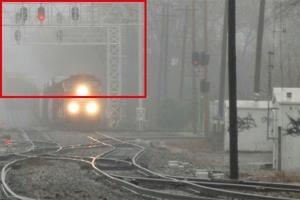}
			&\includegraphics[width=.24\textwidth]{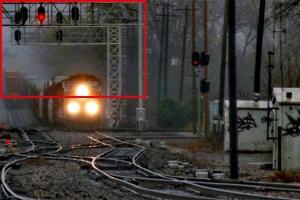}\\
			Input & Berman~\emph{et al.} (2.46) \\
			\includegraphics[width=.24\textwidth]{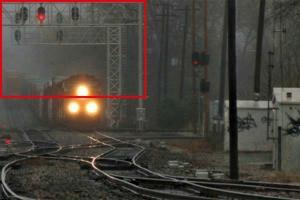}
			&\includegraphics[width=.24\textwidth]{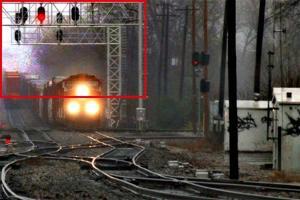} \\
			Zhu~\emph{et al.} (2.73)& Ours (\textbf{2.37})
		\end{tabular}
	\end{center}
	\caption{Comparison with Berman~\emph{et al.} and Zhu~\emph{et al.} (i.e., distribution priors) on real hazy image. The non-reference NIQE scores are reported in brackets below each image.	}\label{fig:distribution}
\end{figure}

\begin{figure}[!htbp]
	\begin{center}
		\begin{tabular}{c@{\extracolsep{0.3em}}c@{\extracolsep{0.3em}}c@{\extracolsep{0.3em}}c@{\extracolsep{0.3em}}c}
			\multirow{2}{*}{\includegraphics[width=.15\textwidth]{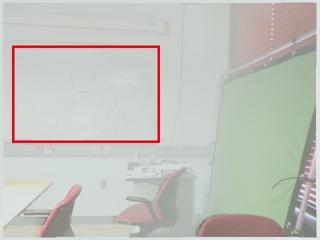}}
			&\includegraphics[width=.15\textwidth]{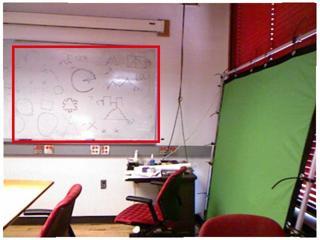}
			&\includegraphics[width=.15\textwidth]{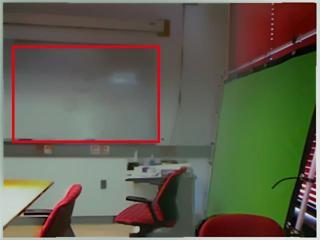}\\
			& Ground Truth  &  Li~\emph{et al.}\\
			&\includegraphics[width=.15\textwidth]{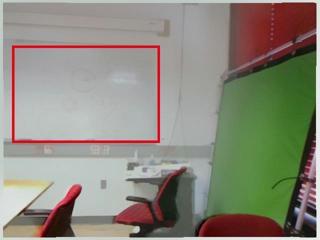}
			&\includegraphics[width=.15\textwidth]{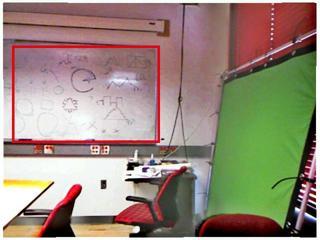}\\
			Input  & Chen~\emph{et al.} & Ours
		\end{tabular}
	\end{center}
	\caption{Comparison with Li~\emph{et al.} and Chen~\emph{et al.} (i.e., variational priors) on synthetic hazy image.}\label{fig:variational}
\end{figure}

\begin{figure*}[!htbp]
	\begin{center}
		\begin{tabular}{c@{\extracolsep{0.3em}}c@{\extracolsep{0.3em}}c@{\extracolsep{0.3em}}c@{\extracolsep{0.3em}}
				c@{\extracolsep{0.3em}}c}
			\includegraphics[width=.19\textwidth]{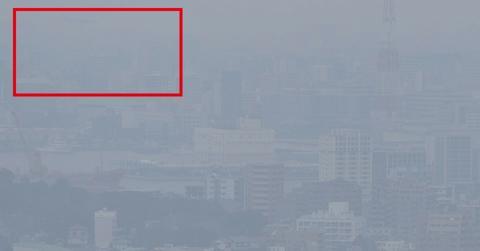}
			&\includegraphics[width=.19\textwidth]{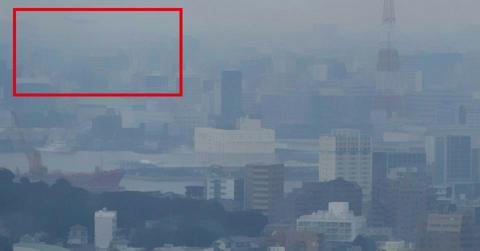}
			&\includegraphics[width=.19\textwidth]{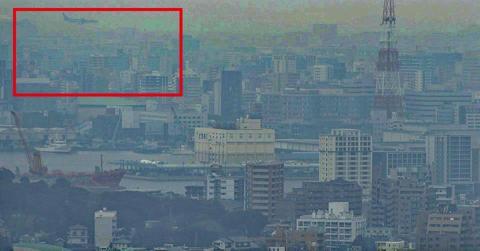}
			&\includegraphics[width=.19\textwidth]{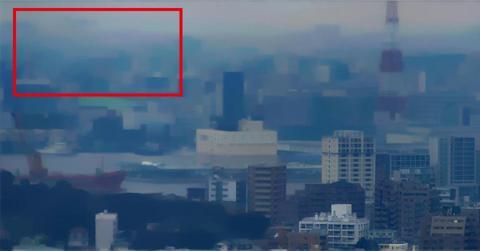}
			&\includegraphics[width=.19\textwidth]{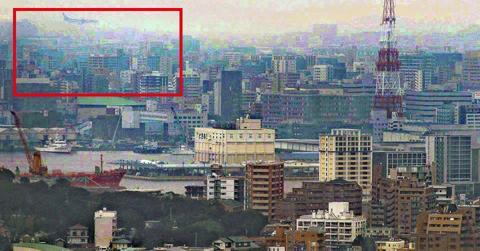}\\
			Input& Chen~\emph{et al.} (5.17)& Meng~\emph{et al.} (3.36) & Li~\emph{et al.} (4.83)& Ours (\textbf{3.16})
		\end{tabular}
		\caption{Comparison with prior based methods on real-world, low-resolution outdoor scene image with extremely dense hazes. The non-reference NIQE scores are reported in brackets below each image. }\label{fig:more}
	\end{center}
\end{figure*}
\begin{figure*}[!htbp]
	\begin{center}
		\begin{tabular}{c@{\extracolsep{0.3em}}c@{\extracolsep{0.3em}}
				c@{\extracolsep{0.3em}}c}
			\includegraphics[width=.24\textwidth]{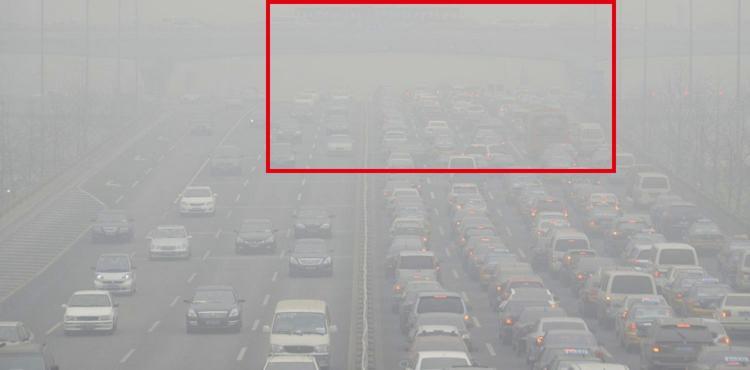}
			&\includegraphics[width=.24\textwidth]{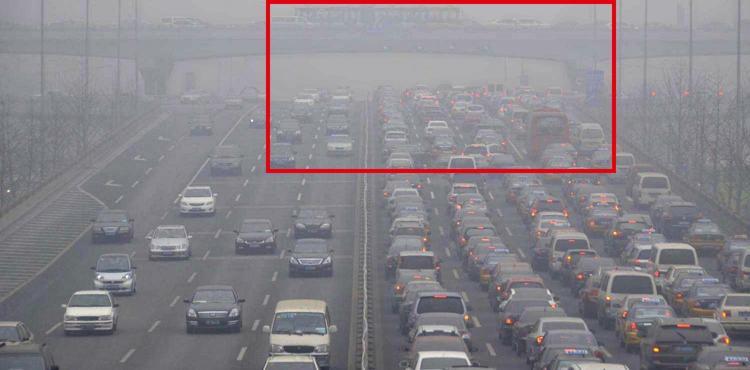}
			&\includegraphics[width=.24\textwidth]{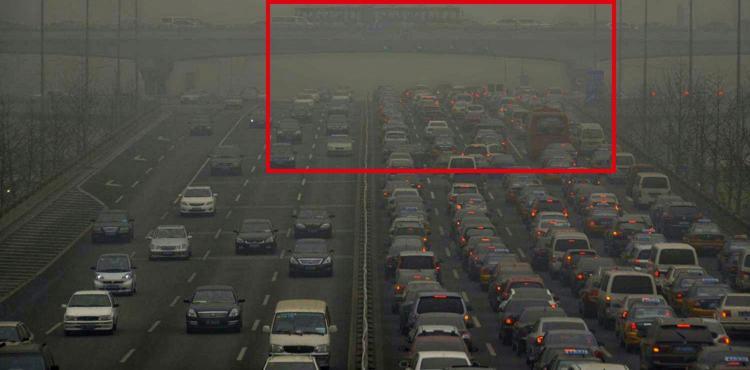}
			&\includegraphics[width=.24\textwidth]{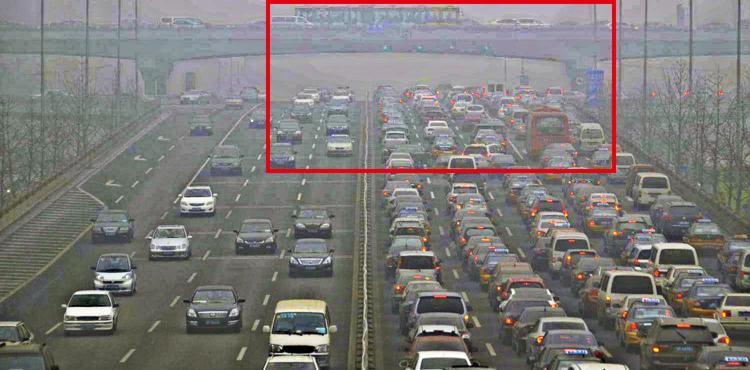}\\
			\includegraphics[width=.24\textwidth]{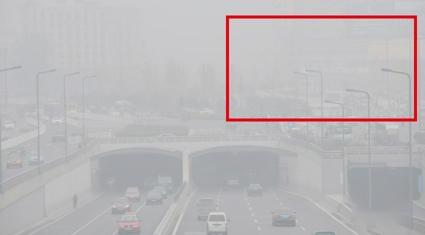}
			&\includegraphics[width=.24\textwidth]{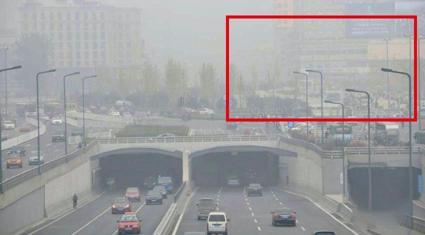}
			&\includegraphics[width=.24\textwidth]{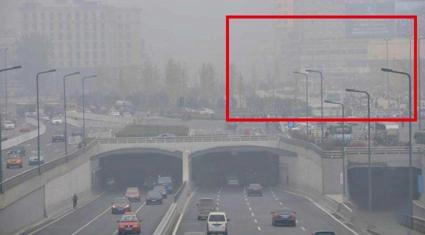}
			&\includegraphics[width=.24\textwidth]{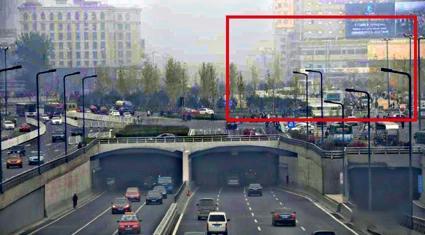}\\
			Input & Ren~\emph{et al.} (3.74 \& 3.04)& Cai~\emph{et al.} (3.74 \& 3.29)& Ours (\textbf{3.53 \& 2.96})
		\end{tabular}
	\end{center}
	\caption{Comparison with Ren~\emph{et al.} and Cai~\emph{et al.} (i.e., data driven deep networks)
		on outdoor scenes with dense hazes. The non-reference NIQE scores are reported in brackets, in which the first and second scores are for results on the top and second rows, respectively.}\label{fig:deep}
\end{figure*}

\subsection{Qualitative Results}
This part conducts experiments on both synthesis and real-world hazy images to
compare performances of our DPATN against state-of-the-art dehazing techniques in qualitative way.
Please notice the detailed comparisons on zoomed in regions in Figures~\ref{fig:physical}, \ref{fig:distribution}, \ref{fig:variational}, \ref{fig:more} and \ref{fig:deep}.

\textbf{Prior Driven Models:} We first compare DPATN with the well-known physical principle based dark channel prior (DCP)~\cite{he2011single} and its recent improvement~\cite{meng2013efficient}) in Figure~\ref{fig:physical}.
We observe that results of DCP have apparent color distortions and serious artifacts. This is because
pixels in sky regions may have high intensity in all three color channels, leading to inexact transmission estimations (see bottom row of Figure~\ref{fig:physical}). In contrast, our DPATN estimates more accurate transmission, thus the restored image is closer to GT. In Figure~\ref{fig:distribution}, we also make a visual comparison on a real hazy image with Berman~\emph{et al.}~\cite{berman2016non} and Zhu~\emph{et al.}~\cite{zhu2015fast}, which can be categorized as distribution prior based approaches.
These two algorithms may not work well on distant region (see sky, lights and trees in red rectangle). While our method successfully removes hazes from the image. Both Li~\emph{et al.}~\cite{li2014contrast} and Chen \emph{et al.}~\cite{chen2016robust} aim to design variational energies to suppress artifacts for haze removal. But we see in Figure~\ref{fig:variational}
that these methods tend to over-smooth local details in dehazing results (see
the white board region). Thanks to the aggregation of both priors and data, DPATN can successfully remove hazes and
preserve rich details in latent image. Finally, we demonstrate results on real-world outdoor scene with
extremely dense hazes in Figure~\ref{fig:more}. It can be seen that most prior-based methods fail on this challenging
example. Though there still exist small artifacts due to low-resolution nature of the input image, details in distant region (e.g., airplane in red rectangle) can still be restored by DPATN.

\begin{table*}[t]
	\caption{Average PSNR / SSIM on the benchmark datasets (i.e., Fattal's benchmark and the newly generated one). As for our dataset (denoted as ``Synth''), we
		denote ``T'', ``M'', ``D'' and ``A'' as subsets of test images with thin, medium, dense hazes and all
		these hazy images, respectively. The numbers of images are listed in the brackets. The best two results are shown in \textcolor[rgb]{1.00,0.00,0.00}{\textbf{red}} and \textcolor[rgb]{0.00,1.00,0.50}{\textbf{green}} fonts.}\label{tab:FHPO2}
	\begin{center}
		\vspace{-5pt}
		\begin{tabular}{|c|c|c|c|c|c|c|c|c|}\hline
			\multicolumn{2}{|c|}{Data}&He~\emph{et al.}&Zhu~\emph{et al.}&Berman~\emph{et al.}&Cai~\emph{et al.}&Ren~\emph{et al.} &Meng~\emph{et al.}&Ours \\
			\hline\hline
			\multicolumn{2}{|c|}{Fattal \cite{fattal2014dehazing} (\#11)}&13.62 / 0.77&\textcolor{green}{\textbf{16.66}} / \textcolor{green}{\textbf{0.80}}&14.94 / 0.79&15.76 / 0.79&15.42 / \textcolor{green}{\textbf{0.80}}&14.30 / 0.79 &\textcolor{red}{\textbf{17.91}} / \textcolor{red}{\textbf{0.86}}\\
			\hline\hline
			\multirow{4}{*}{Synth}
			&T (\#60)&12.32 / \textcolor{green}{\textbf{0.75}}&14.21 / 0.72&15.69 / 0.71&\textcolor{red}{\textbf{16.15}} / \textcolor{red}{\textbf{0.76}}&14.11 / 0.70&13.95 / 0.64&\textcolor{green}{\textbf{15.89}} / \textcolor{red}{\textbf{0.76}}\\\cline{2-9}
			&M (\#60)&11.00 / 0.60& 12.75 / \textcolor{green}{\textbf{0.67}}& 10.20 / 0.60& \textcolor{green}{\textbf{12.79}} / \textcolor{green}{\textbf{0.67}}&10.10 / 0.59&9.21 / 0.57&\textcolor{red}{\textbf{15.78}} / \textcolor{red}{\textbf{0.74}}\\\cline{2-9}
			&D (\#60)&10.31 / 0.51&\textcolor{green}{\textbf{11.8}} / \textcolor{green}{\textbf{0.58}} & 9.49 / 0.49 & 11.70 / 0.56&9.47 / 0.49&9.25 / 0.52&\textcolor{red}{\textbf{15.42}} / \textcolor{red}{\textbf{0.68}}\\\cline{2-9}
			&A (\#180)&11.21 / 0.62&12.95 / 0.66&11.79 / 0.60&\textcolor{green}{\textbf{13.55}} / \textcolor{green}{\textbf{0.67}}&11.23 / 0.59&10.80 / 0.58&\textcolor{red}{\textbf{15.70}} / \textcolor{red}{\textbf{0.73}}\\
			\hline
		\end{tabular}
	\end{center}
\end{table*}

\begin{table*}[!htbp]
	\caption{Average running time (seconds per image) of dehazing methods in testing phase on Fattal's benchmark.}\label{tab:time}
	\begin{center}
		\vspace{-5pt}
		\begin{tabular}{|c|c|c|c|c|c|c|c|c|c|}\hline
			Method&He~\emph{et al.}&Zhu~\emph{et al.}&Berman~\emph{et al.}&Cai~\emph{et al.}&Ren~\emph{et al.} &Meng~\emph{et al.}& Li~\emph{et al.} & Chen~\emph{et al.} &Ours \\
			\hline
			{Time}&17.20s&4.38s&3.73s&5.77s&5.87s&6.95s & 83.44s & 105.13s &6.46s\\
			\hline
		\end{tabular}
	\end{center}
	
\end{table*}

\begin{table*}[!htbp]
	\caption{The quantitative results (i.e., PSNR / SSIM) of DPATN with different training data sizes (i.e., \#30, \#50, \#80) on five example images in Fattal's benchmark. The results of two CNNs based methods (i.e., Ren~\emph{et al.} and Cai~\emph{et al.}) are also reported at the bottom two rows. The best two results are shown in \textcolor[rgb]{1.00,0.00,0.00}{\textbf{red}} and \textcolor[rgb]{0.00,1.00,0.50}{\textbf{blue}} fonts, respectively.}\label{tab:differ_train}
	\begin{center}
		\vspace{-5pt}
		\begin{tabular}{|c|c|c|c|c|c|c|}\hline
			& \emph{church}  & \emph{road1} & \emph{lawn2} & \emph{mansion} & \emph{raindeer} & Average\\ \hline\hline
			Ours (\#30) &15.74 / 0.78&12.59 / 0.68&14.01 / 0.71&16.59 / 0.78&14.29 / 0.76&14.64 / 0.74\\ \hline
			Ours (\#50)&\textcolor{red}{\textbf{18.38}} / \textcolor{red}{\textbf{0.89}}&\textcolor{green}{\textbf{18.16}} / \textcolor{green}{\textbf{0.0.86}}&\textcolor{green}{\textbf{17.43}} / \textcolor{red}{\textbf{0.86}}&\textcolor{red}{\textbf{21.36}} / \textcolor{red}{\textbf{0.94}}&\textcolor{green}{\textbf{18.95}} / 0.80&\textcolor{red}{\textbf{18.86}} / \textcolor{green}{\textbf{0.87}}\\\hline
			Ours (\#80)&\textcolor{green}{\textbf{18.24}} / \textcolor{red}{\textbf{0.89}}&\textcolor{red}{\textbf{18.48}} / \textcolor{red}{\textbf{0.87}}&\textcolor{red}{\textbf{18.15}} / \textcolor{green}{\textbf{0.85}}& \textcolor{green}{\textbf{19.90}} / \textcolor{green}{\textbf{0.93}} & \textcolor{red}{\textbf{19.27}} / \textcolor{red}{\textbf{0.86}}&\textcolor{green}{\textbf{18.81}} / \textcolor{red}{\textbf{0.88}}\\\hline\hline
			Cai~\emph{et al.}&14.98 / 0.78 & 13.62 / 0.73 &13.34 / 0.72 & 16.99 / 0.85 & 18.23 / \textcolor{green}{\textbf{0.82}}&15.43 / 0.78\\\hline
			Ren~\emph{et al.}& 15.01 / \textcolor{green}{\textbf{0.84}} &14.17 / 0.78&14.29 / 0.78&17.97 / 0.89&17.80 / 0.81&15.85 / 0.82\\
			\hline
		\end{tabular}
	\end{center}
\end{table*}

\textbf{Data Driven Deep Models:}
We then compare our DPATN with recently developed deep models (e.g., Ren \emph{et al.}~\cite{ren2016single} and Cai~\emph{et al.}~\cite{cai2016dehazenet}) for natural outdoor scene dehazing in Figure~\ref{fig:deep}.
The performances of two CNNs are not very good on these challenging real-world scenarios. Actually this is not very surprising because
all these networks are established based on intuitions and fully trained on synthetic hazy images. So they indeed completely
ignore priors of hazes and thus the estimated transmissions will have bias if hazes in testing images have
significantly different distributions with training data. In contrast, the propagation of DPATN are not only learned on training data (i.e., $\mathcal{D}$), but also controlled by an \emph{adaptive} domain knowledge based guidance (i.e., $\mathcal{P}$),
thus we actually leverage advantages of different dehazing methodologies. We can see that DPATN can well
remove most hazes and produce a clear scene with vivid color information and rich details, which verify the
efficiency of the proposed aggregation strategy.

\subsection{Quantitative Results}

\textbf{Benchmarks:}
We evaluate DPATN and report quantitative results on different benchmark datasets.
The first one is from Fattel~\cite{fattal2014dehazing} and has been widely used for dehazing evaluation. It contains 11 images, including architecture, natural scenery and indoor images.
We then generate a larger testing dataset with 180 synthetic hazy images using NYU depth database
(completely different from our training images).
Three subsets are further selected from this dataset based on the haze concentration (i.e., thin, medium and dense, each consists of 60 images).

\textbf{Performance Evaluation:}
We report average PSNR and SSIM on two benchmarks in Table~\ref{tab:FHPO2}.
Prior based methods perform relatively well on Fattal's small dataset. In contrast,
the results of deep models are better than conventional ones on subset with thin haze concentrations. For example, the method of Cai~\emph{et al.} achieves the second best results in this case. This is because these CNNs are trained on synthetic images. So they may perform well when testing images share similar haze distributions with their training data. Thanks to the aggregation of data and priors, DPATN archives very good performance on all these benchmark datasets.

\textbf{Running Time:} We report average running time of dehazing methods in testing phase on Fattal's database
in Table~\ref{tab:time}.
One can see that the speed of DPATN is faster than most prior based methods, but a little bit slower than existing CNNs (e.g., Ren~\emph{et al.} and Cai~\emph{et al.}).
This is mainly because their networks are shallower and thus have less convolution operators.
But notice that even with deeper architectures, the number of parameters in DPATN is still extremely less than existing fully data driven deep models. So our training cost is definitely much lower than these CNNs.

\subsection{Ablation Studies}\label{sec:data}

Now we provide analysis on the training phase of DPATN in detail as follows.

\textbf{Training Data Size:}
Generally, existing dehazing networks all require large amounts of data for training.
For example, Ren~\emph{et al.} collected 6,000 images and generated transmissions by 3 different medium extinction coefficients, resulting in
18,000 hazy images. Similarly, dehazing network proposed by Cai~\emph{et al.} required 20,000 training images.
In contrast, by incorporating rich priors into the network, DPATN can yield accurate results with only dozens of training images. To verify this issue, we report quantitative performances of DPATN with different training data size in Table~\ref{tab:differ_train}.
It is observed that DPATN with 50 training pairs performs much better than the 30 case.
However, we cannot achieve significant improvements even when data size increases to 80. Overall,
50 training pairs are sufficient and we follow this setting in all experiments.

\begin{figure}[htp]
	\begin{center}
		\begin{tabular}{c@{\extracolsep{0em}}c}
			\includegraphics[width=0.24\textwidth,
			keepaspectratio]{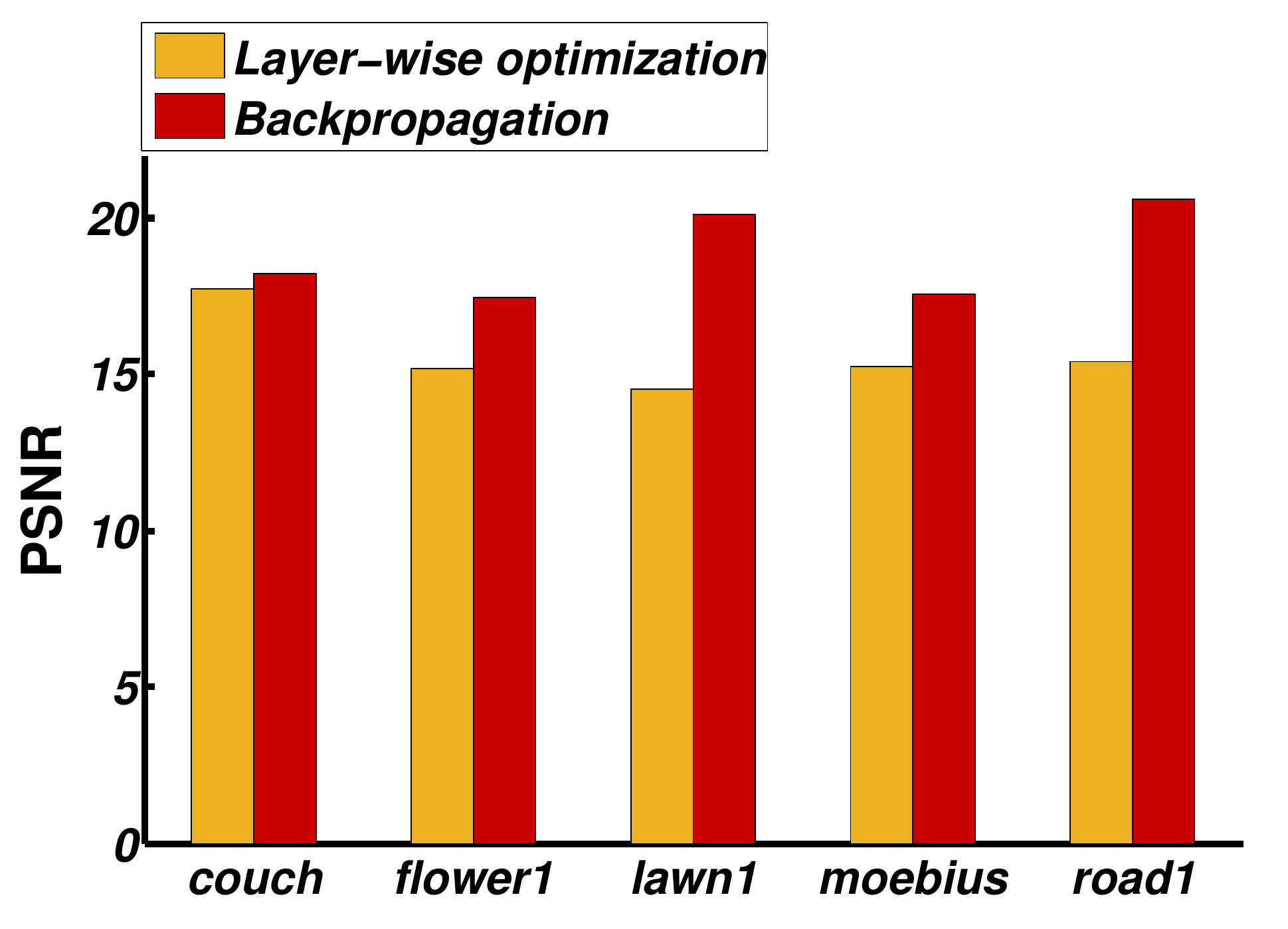}&
			\includegraphics[width=0.24\textwidth,
			keepaspectratio]{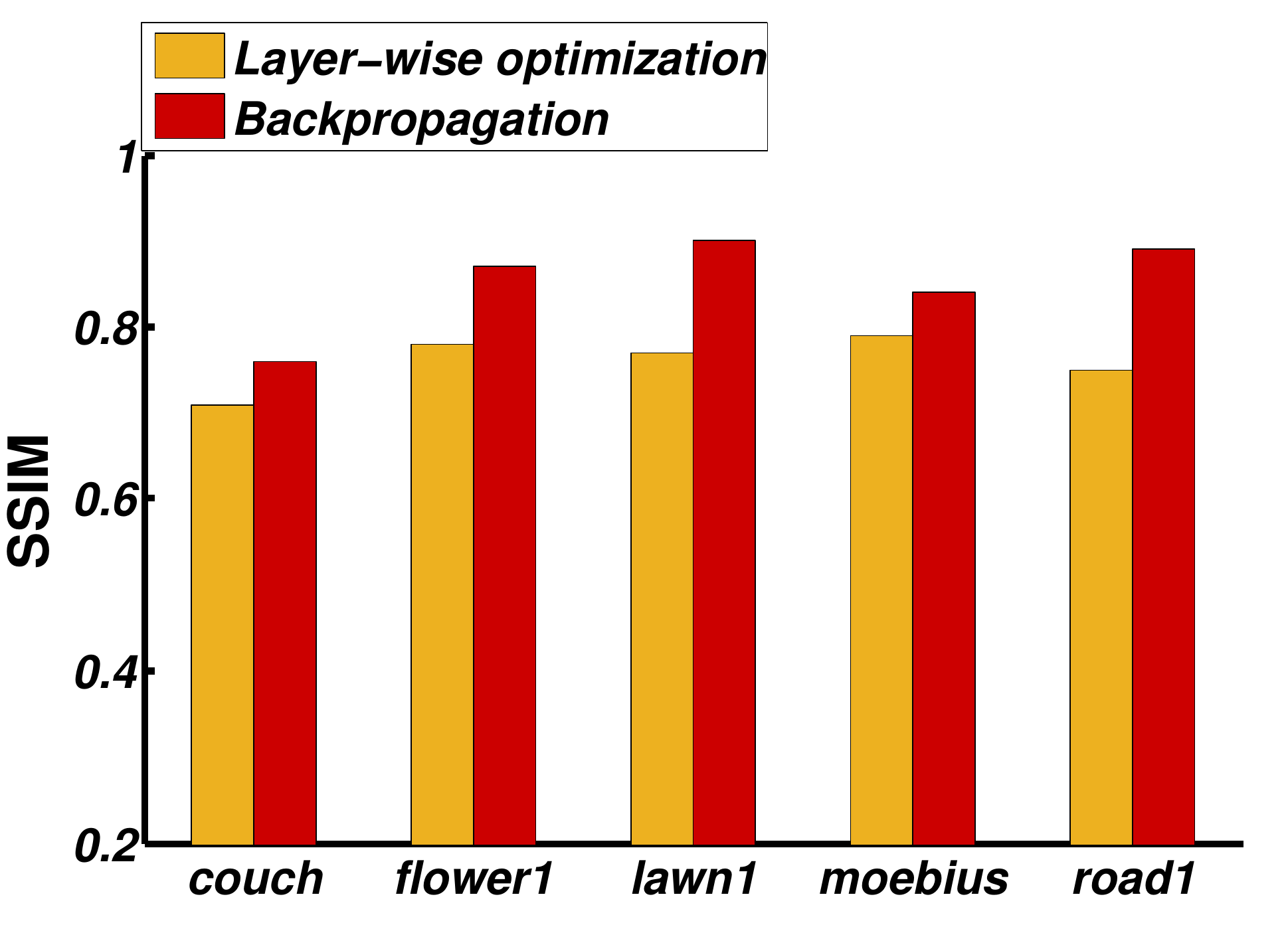}\\
		\end{tabular}
	\end{center}
	\caption{The performances of DPATN with greedy layer-wise optimization and joint backpropagation on example images
		(i.e., \emph{couch}, \emph{flower1}, \emph{lawn1}, \emph{moebius}, \emph{road1}) in Fattal's benchmark.}\label{fig:greedy_joint}
\end{figure}
\textbf{Optimization Strategy:}
In this experiment, we compare two different optimization strategies for DPATN training in practice, i.e., layer-wise optimization and
joint backpropagation. As for the first scheme, we simply calculate gradients of training cost with respect to parameters of each layer separately.
In contrast, the second candidate scheme performs joint gradient descent updating for all the parameters.
The PSNR and SSIM scores of DPATN with two training
strategies are illustrated in Figure~\ref{fig:greedy_joint}. One can observe that the joint strategy
(with layer-wise initialization)
outperforms the greedy one in all test examples. So in this paper we adopt the former strategy for the training of DPATN.

\subsection{Underwater Image Enhancement}

In this part, we test the proposed framework on a more challenging underwater image enhancement task.

\textbf{Training Set:}
Different from standard haze removal problem, which utilize uniformed scattering component, here we follow Eq.~\eqref{eq:underwater}
to generate underwater training data based on NYU depth database. Specifically, as light with different color has different attenuation degree, we generate underwater images by adjusting the background light $\mathbf{B}_r \in [0.05, 0.2]$, $\mathbf{B}_g \in [0.6, 0.8]$, $\mathbf{B}_b \in [0.7, 1.0]$, and medium extinction coefficients $\beta_r\in [0.05, 0.15]$, $\beta_g\in [0.6, 0.9]$, $\beta_b\in [0.7, 1.0]$ in Eq.~\eqref{eq:underwater}. We then crop a $180 \times 180$ region from each transmission map to build our training set.

\textbf{Performance Evaluation:} We first evaluate the efficiency of image separation component for our proposed framework. That is, we compare the performance of DPATN with and without task-aware image separation (i.e., Eq.~\eqref{eq:task-aware}) in
Figure~\ref{fig:preprocess}. The zoomed in results are shown in Figure~\ref{fig:preprocess detail}.
It can be observed that DPATN with task-aware separation can suppress the influence of forward scattering component and remove the haze effect at the edge of the object. Moreover, it also balances the image colors of the final results.
We then compare our complete framework with existing state-of-the-art underwater enhancement algorithms (e.g., Rahman~\emph{et al.}~\cite{rahman1996multi},
Chiang~\emph{et al.}~\cite{Chiang2012restore} and Li~\emph{et al.}~\cite{li2014contrast}) on several challenging degraded underwater
images\footnote{All images are downloaded from \url{https://github.com/agaldran/UnderWater}.}.
It can be seen from Figure~\ref{fig:underwater} and the zoomed in results in Figure~\ref{fig:underwater detail} that though~\cite{rahman1996multi} can remove the blue and green tones well, the haze effect still exists in its results. And~\cite{Chiang2012restore} is very sensitive to the parameters. We can also see that~\cite{li2014contrast} tends to over-smooth local details, especially in the scene far away from the camera. In contrast, our proposed method can simultaneously recover image details and balance the image colors.

	\begin{figure}[!htbp]
	\begin{center}
		\begin{tabular}{c@{\extracolsep{0.3em}}c@{\extracolsep{0.3em}}c@{\extracolsep{0.3em}}
				c@{\extracolsep{0.3em}}c}
			\includegraphics[width=.15\textwidth]{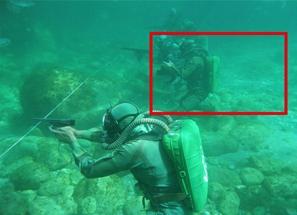}
			&\includegraphics[width=.15\textwidth]{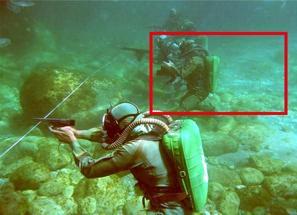}
			&\includegraphics[width=.15\textwidth]{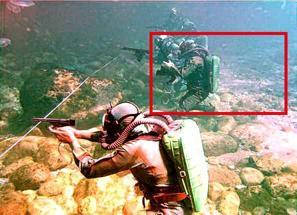}\\
			\includegraphics[width=.15\textwidth]{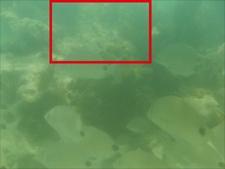}
			&\includegraphics[width=.15\textwidth]{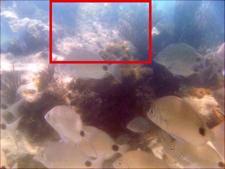}
			&\includegraphics[width=.15\textwidth]{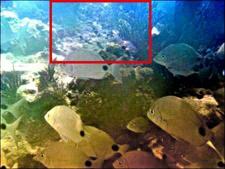}\\
			Input & Ours (WO) & Ours (W)	
		\end{tabular}
	\end{center}
	\caption{The performances of our framework without and with task-aware image separation (respectively denoted as ``WO'' and ``W'') on underwater degraded images.}\label{fig:preprocess}
\end{figure}

\begin{figure}[!htbp]
	\begin{center}
		\begin{tabular}{c@{\extracolsep{0.3em}}c@{\extracolsep{0.3em}}c@{\extracolsep{0.3em}}
				c@{\extracolsep{0.3em}}c}
			\includegraphics[width=.15\textwidth]{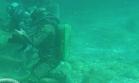}
			&\includegraphics[width=.15\textwidth]{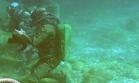}
			&\includegraphics[width=.15\textwidth]{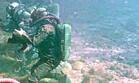}\\
			\includegraphics[width=.15\textwidth]{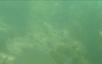}
			&\includegraphics[width=.15\textwidth]{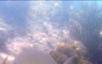}
			&\includegraphics[width=.15\textwidth]{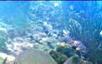}\\
			Input & Ours (WO) & Ours (W)
		\end{tabular}
	\end{center}
	\caption{Zoomed in comparisons on underwater degraded images in Figure~\ref{fig:preprocess}.}\label{fig:preprocess detail}
\end{figure}

\begin{figure*}[t]
	\begin{center}
		\begin{tabular}{c@{\extracolsep{0.2em}}c@{\extracolsep{0.2em}}c@{\extracolsep{0.2em}}c@{\extracolsep{0.2em}}c}
			\includegraphics[width=.18\textwidth]{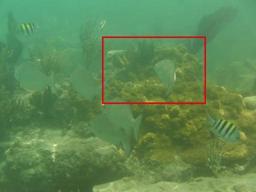}
			&\includegraphics[width=.18\textwidth]{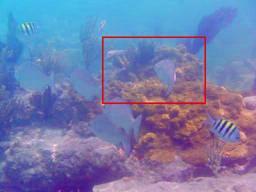}
			&\includegraphics[width=.18\textwidth]{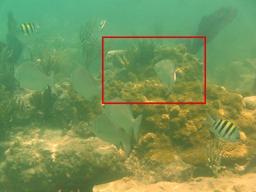}
			&\includegraphics[width=.18\textwidth]{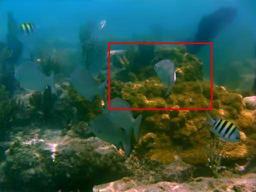}
			&\includegraphics[width=.18\textwidth]{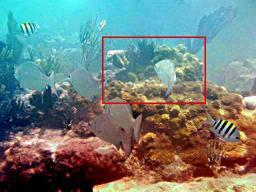}
			\\
			\includegraphics[width=.18\textwidth]{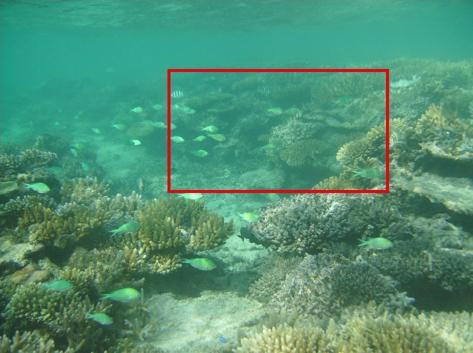}
			&\includegraphics[width=.18\textwidth]{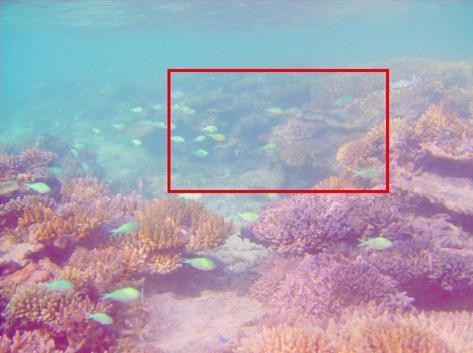}
			&\includegraphics[width=.18\textwidth]{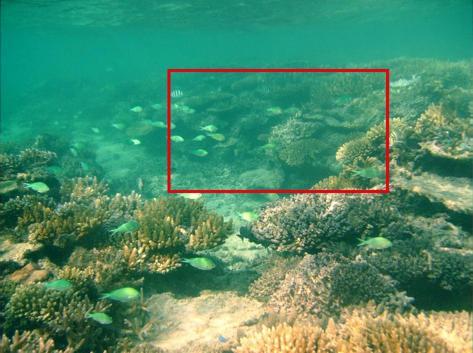}
			&\includegraphics[width=.18\textwidth]{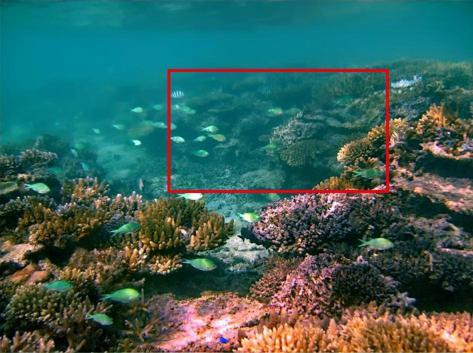}
			&\includegraphics[width=.18\textwidth]{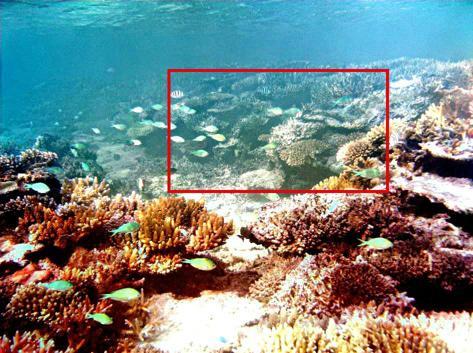}
			\\
			\includegraphics[width=.18\textwidth]{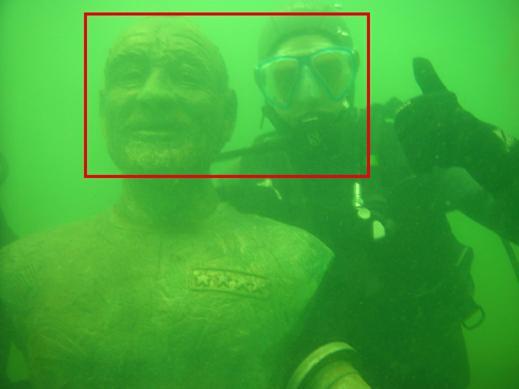}
			&\includegraphics[width=.18\textwidth]{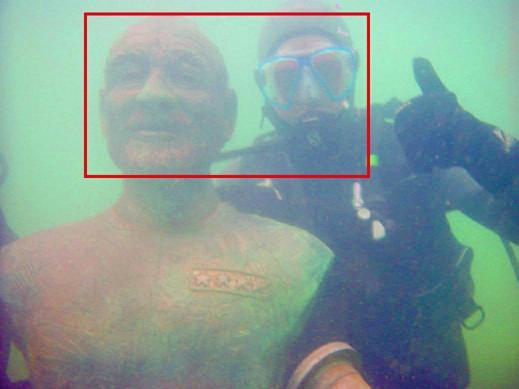}
			&\includegraphics[width=.18\textwidth]{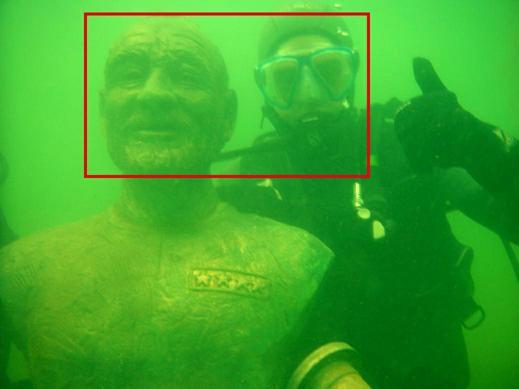}
			&\includegraphics[width=.18\textwidth]{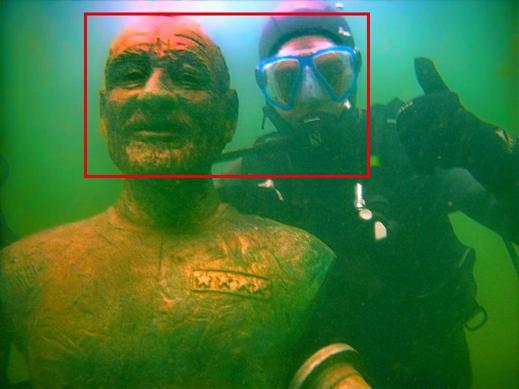}
			&\includegraphics[width=.18\textwidth]{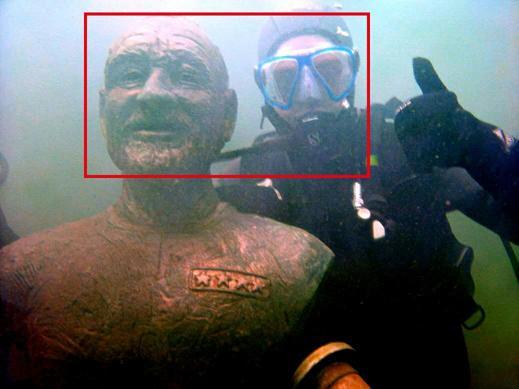}
			\\
			\includegraphics[width=.18\textwidth]{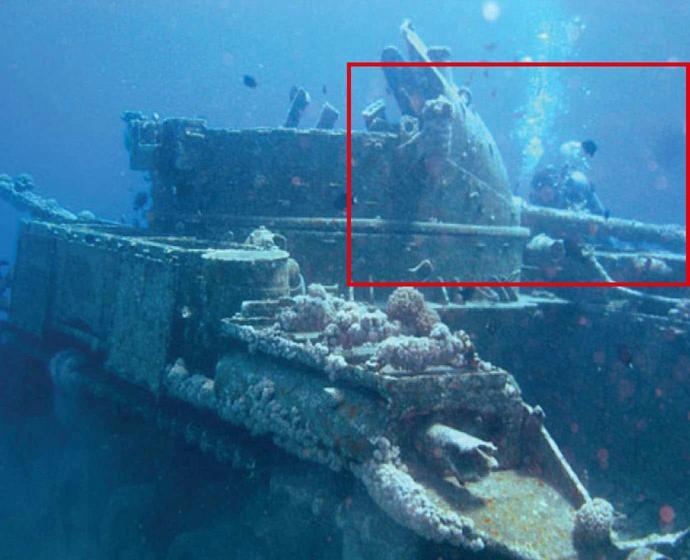}
			&\includegraphics[width=.18\textwidth]{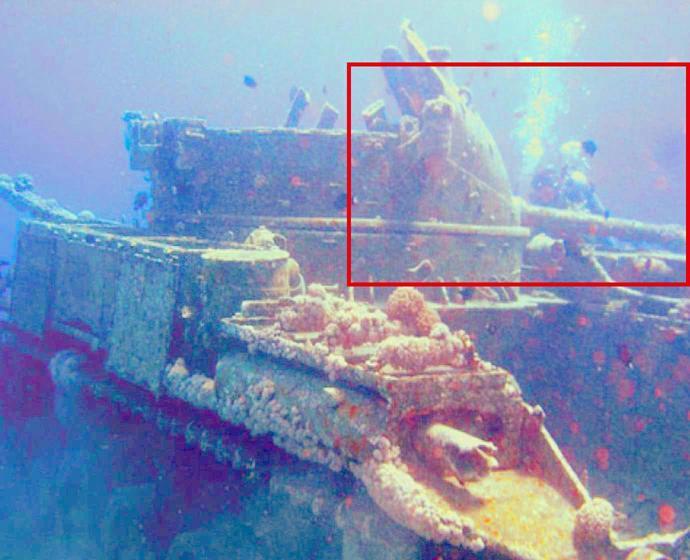}
			&\includegraphics[width=.18\textwidth]{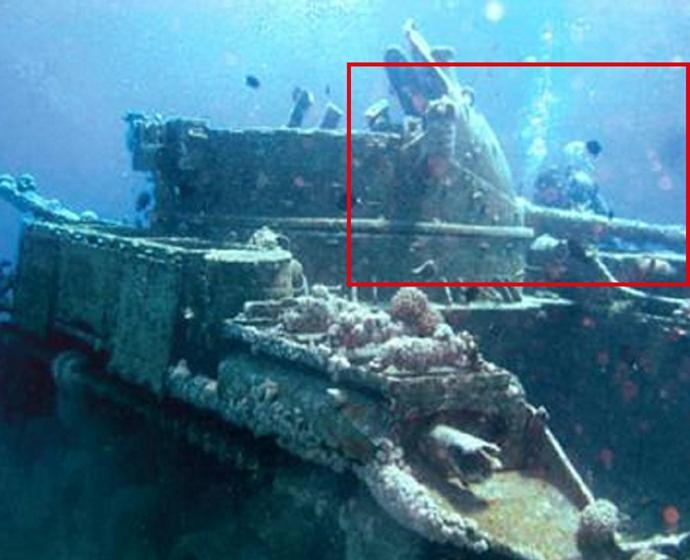}
			&\includegraphics[width=.18\textwidth]{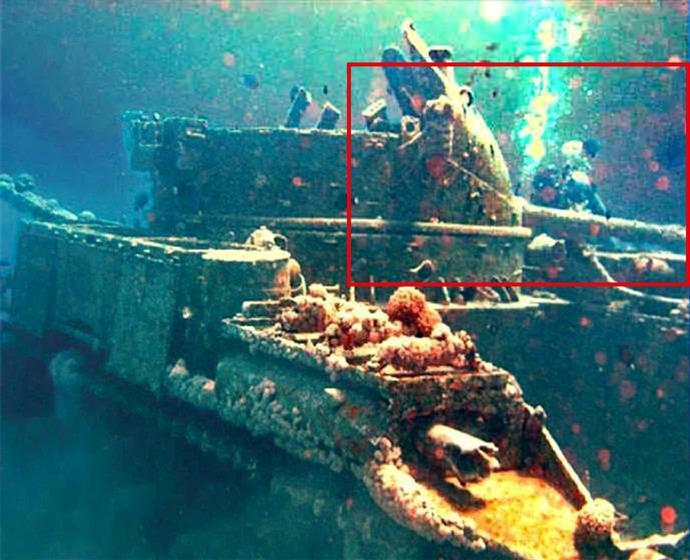}
			&\includegraphics[width=.18\textwidth]{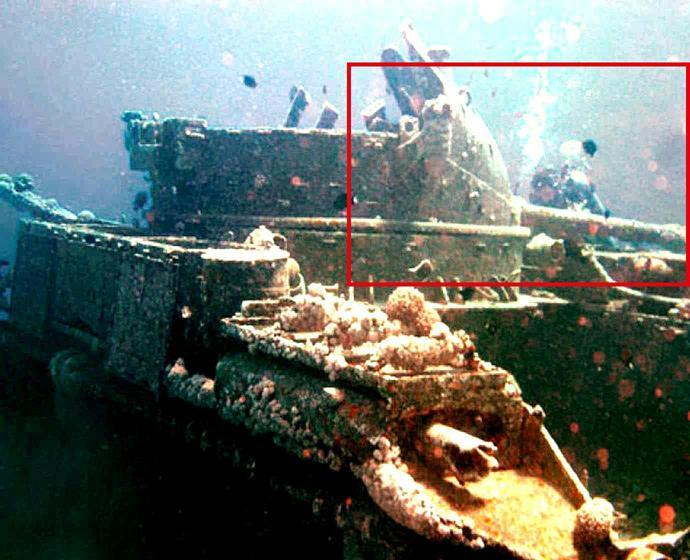}
			\\Input&Rahman~\emph{et al.}&Chiang~\emph{et al.}&Li~\emph{et al.} &Ours
		\end{tabular}
	\end{center}
	\caption{Comparison with Rahman~\emph{et al.}, Chiang~\emph{et al.} and Li~\emph{et al.} on underwater degraded images.}\label{fig:underwater}
\end{figure*}
\begin{figure*}[!htbp]
	\begin{center}
		\begin{tabular}{c@{\extracolsep{0.2em}}c@{\extracolsep{0.2em}}c@{\extracolsep{0.2em}}c@{\extracolsep{0.2em}}c}
			\includegraphics[width=.18\textwidth]{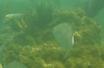}
			&\includegraphics[width=.18\textwidth]{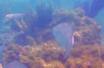}
			&\includegraphics[width=.18\textwidth]{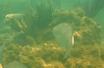}
			&\includegraphics[width=.18\textwidth]{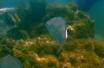}
			&\includegraphics[width=.18\textwidth]{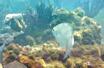}
			\\
			\includegraphics[width=.18\textwidth]{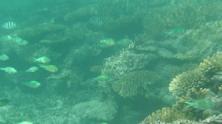}
			&\includegraphics[width=.18\textwidth]{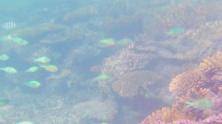}
			&\includegraphics[width=.18\textwidth]{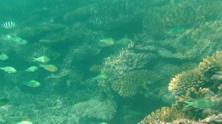}
			&\includegraphics[width=.18\textwidth]{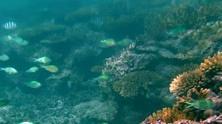}
			&\includegraphics[width=.18\textwidth]{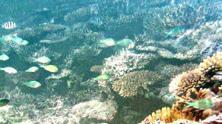}
			\\
			\includegraphics[width=.18\textwidth]{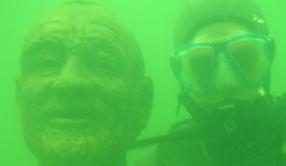}
			&\includegraphics[width=.18\textwidth]{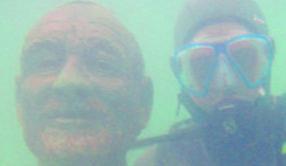}
			&\includegraphics[width=.18\textwidth]{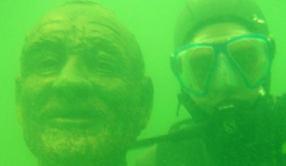}
			&\includegraphics[width=.18\textwidth]{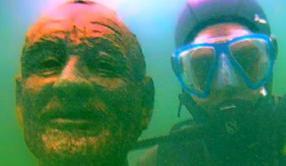}
			&\includegraphics[width=.18\textwidth]{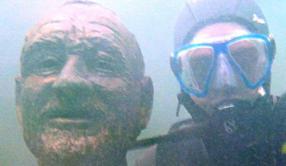}
			\\
			\includegraphics[width=.18\textwidth]{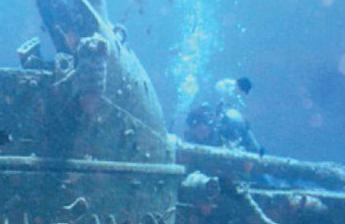}
			&\includegraphics[width=.18\textwidth]{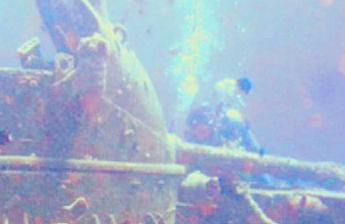}
			&\includegraphics[width=.18\textwidth]{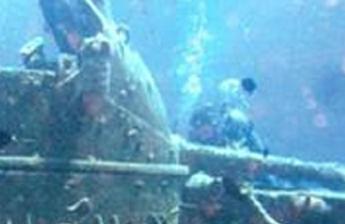}
			&\includegraphics[width=.18\textwidth]{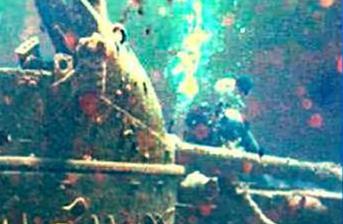}
			&\includegraphics[width=.18\textwidth]{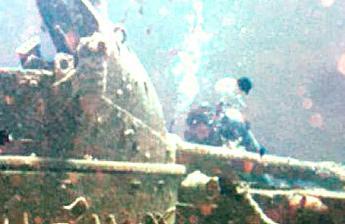}
			\\Input&Rahman~\emph{et al.}&Chiang~\emph{et al.}&Li~\emph{et al.} &Ours
		\end{tabular}
	\end{center}
	\caption{Zoomed in comparison with Rahman~\emph{et al.}, Chiang~\emph{et al.} and Li~\emph{et al.} on underwater degraded images in Figure~\ref{fig:underwater}.}\label{fig:underwater detail}
\end{figure*}

\subsection{Single Image Rain Removal}

	\begin{figure*}[htb]
	\centering
	\begin{tabular}{c@{\extracolsep{0.2em}}c@{\extracolsep{0.2em}}c@{\extracolsep{0.2em}}c@{\extracolsep{0.2em}}c@{\extracolsep{0.2em}}c@{\extracolsep{0.2em}}c@{\extracolsep{0.2em}}c@{\extracolsep{0.2em}}c@{\extracolsep{0.2em}}c}
		
		\includegraphics[width=.16\textwidth]{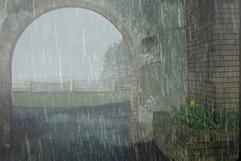}
		&\includegraphics[width=.16\textwidth]{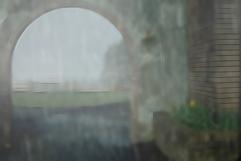}
		&\includegraphics[width=.16\textwidth]{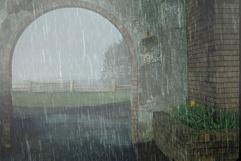}
		&\includegraphics[width=.16\textwidth]{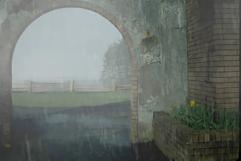}
		&\includegraphics[width=.16\textwidth]{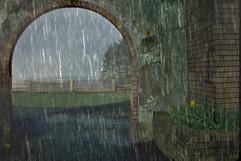}
		&\includegraphics[width=.16\textwidth]{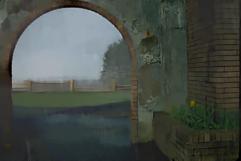}
		\\
		
		\includegraphics[width=.16\textwidth]{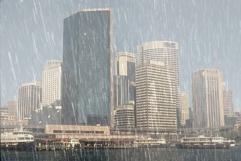}
		&\includegraphics[width=.16\textwidth]{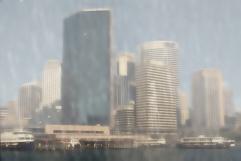}
		&\includegraphics[width=.16\textwidth]{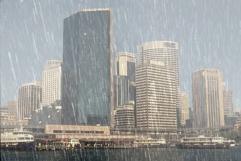}
		&\includegraphics[width=.16\textwidth]{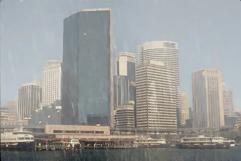}
		&\includegraphics[width=.16\textwidth]{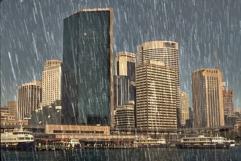}
		&\includegraphics[width=.16\textwidth]{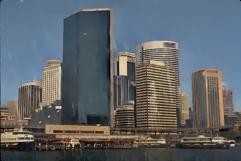}
		\\
		
		\includegraphics[width=.16\textwidth]{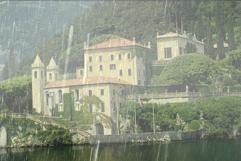}
		&\includegraphics[width=.16\textwidth]{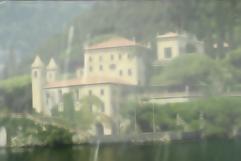}
		&\includegraphics[width=.16\textwidth]{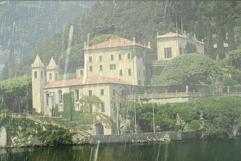}
		&\includegraphics[width=.16\textwidth]{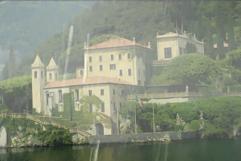}
		&\includegraphics[width=.16\textwidth]{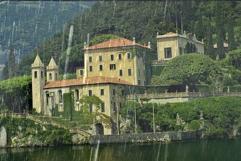}
		&\includegraphics[width=.16\textwidth]{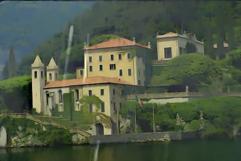}
		\\
		\includegraphics[width=.16\textwidth]{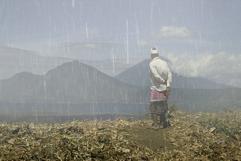}
		&\includegraphics[width=.16\textwidth]{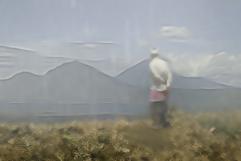}
		&\includegraphics[width=.16\textwidth]{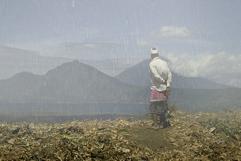}
		&\includegraphics[width=.16\textwidth]{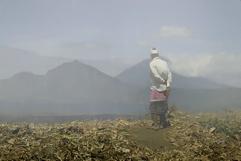}
		&\includegraphics[width=.16\textwidth]{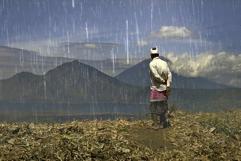}
		&\includegraphics[width=.16\textwidth]{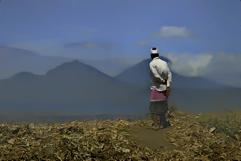}
		\\
		\includegraphics[width=.16\textwidth]{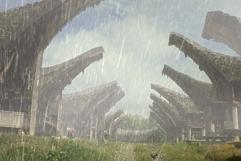}
		&\includegraphics[width=.16\textwidth]{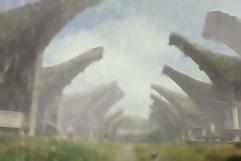}
		&\includegraphics[width=.16\textwidth]{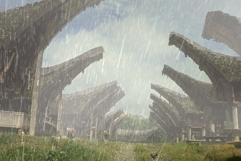}
		&\includegraphics[width=.16\textwidth]{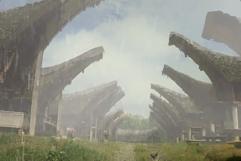}
		&\includegraphics[width=.16\textwidth]{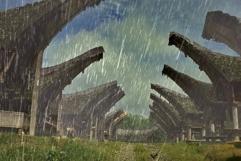}
		&\includegraphics[width=.16\textwidth]{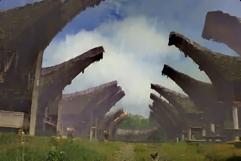}
		\\
		Input & SR &  DSC & LP&  Ours (WO) &  Ours (W) \\

	\end{tabular}
	\caption{Comparison with SR, DSC and LP on images with synthesized rain streaks and hazes. The last two columns are results of our proposed framework without and with task-aware image separation (respectively denoted as ``WO'' and ``W'').}
	\label{fig:visual}
\end{figure*}

\begin{figure*}[!htbp]
	\centering
	\begin{tabular}{c@{\extracolsep{0.2em}}c@{\extracolsep{0.2em}}c@{\extracolsep{0.2em}}c@{\extracolsep{0.2em}}c@{\extracolsep{0.2em}}c@{\extracolsep{0.2em}}c@{\extracolsep{0.2em}}c@{\extracolsep{0.2em}}c@{\extracolsep{0.2em}}c}
		
		\includegraphics[width=.16\textwidth]{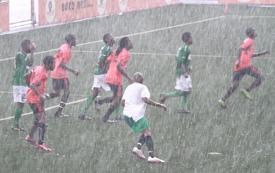}
		&\includegraphics[width=.16\textwidth]{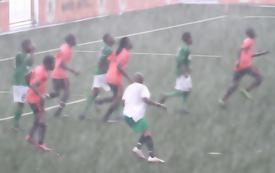}
		&\includegraphics[width=.16\textwidth]{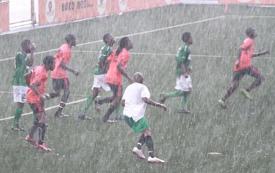}
		&\includegraphics[width=.16\textwidth]{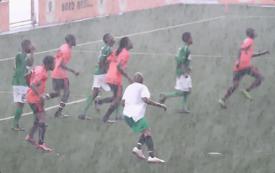}
		&\includegraphics[width=.16\textwidth]{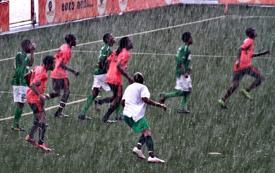}
		&\includegraphics[width=.16\textwidth]{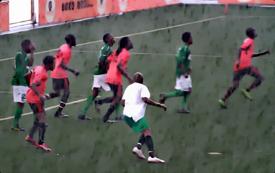}
		\\
		
		\includegraphics[width=.16\textwidth]{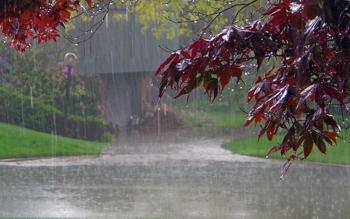}
		&\includegraphics[width=.16\textwidth]{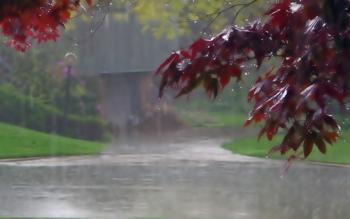}
		&\includegraphics[width=.16\textwidth]{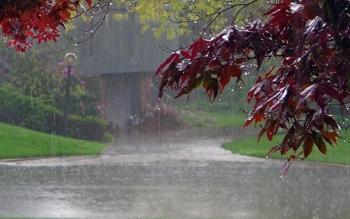}
		&\includegraphics[width=.16\textwidth]{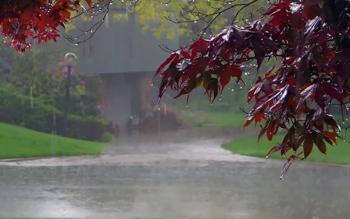}
		&\includegraphics[width=.16\textwidth]{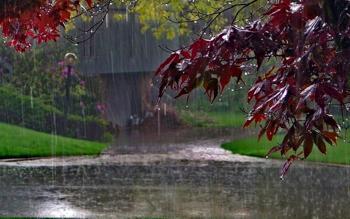}
		&\includegraphics[width=.16\textwidth]{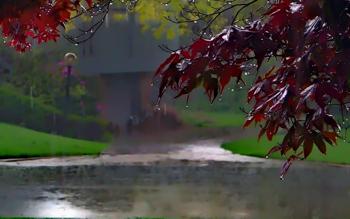}
		\\
		
		\includegraphics[width=.16\textwidth]{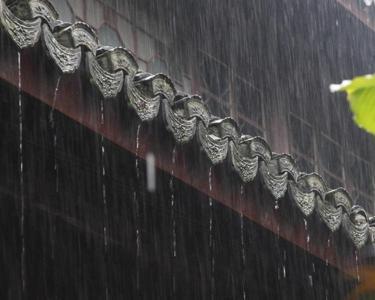}
		&\includegraphics[width=.16\textwidth]{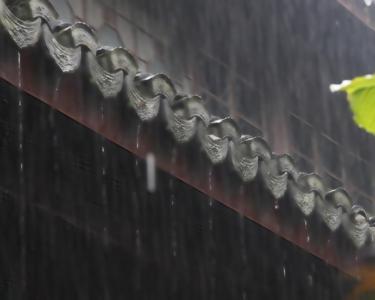}
		&\includegraphics[width=.16\textwidth]{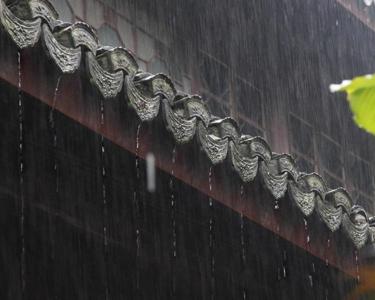}
		&\includegraphics[width=.16\textwidth]{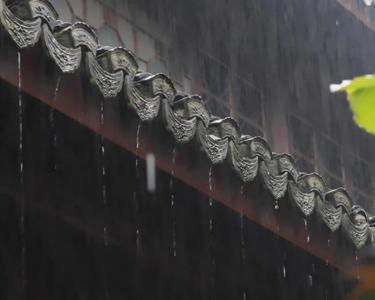}
		&\includegraphics[width=.16\textwidth]{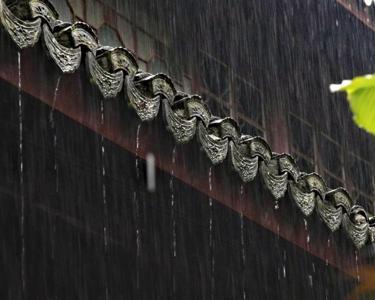}
		&\includegraphics[width=.16\textwidth]{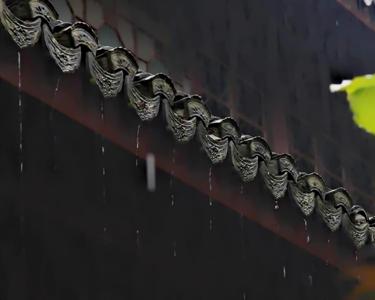}
		\\
		\includegraphics[width=.16\textwidth]{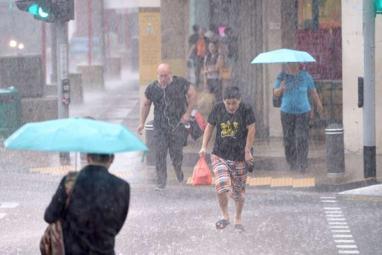}
		&\includegraphics[width=.16\textwidth]{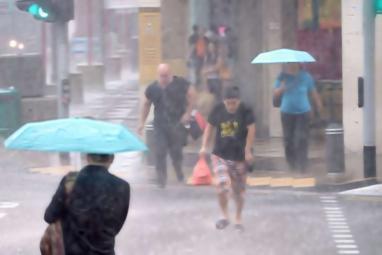}
		&\includegraphics[width=.16\textwidth]{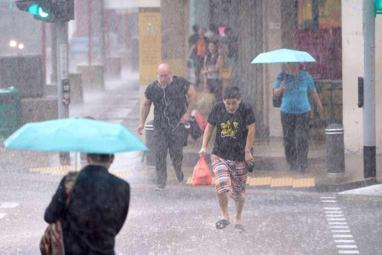}
		&\includegraphics[width=.16\textwidth]{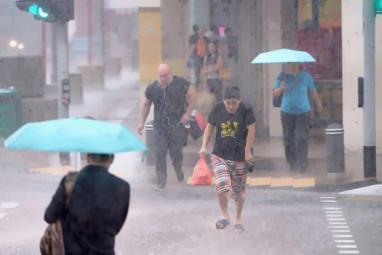}
		&\includegraphics[width=.16\textwidth]{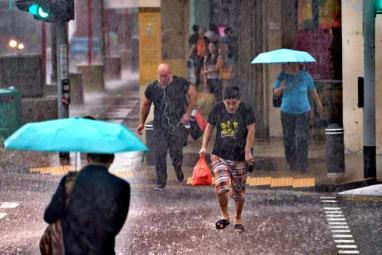}
		&\includegraphics[width=.16\textwidth]{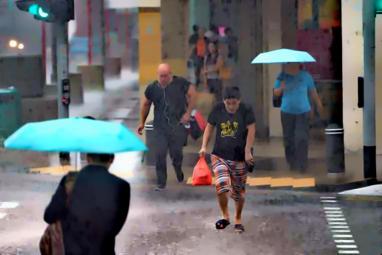}
		
		\\
		Input & SR &  DSC & LP&  Ours (WO)&  Ours (W)\\
	\end{tabular}
	\caption{Comparison with SR, DSC and LP on images with real heavy rain streaks. The last two columns are results of our proposed framework without and with task-aware image separation (respectively denoted as ``WO'' and ``W'').}
	\label{fig:visual-real}
\end{figure*}

Finally, we conduct experiments on rain removal and compare the performance of our proposed framework with state-of-the-art deraining approaches, including sparse representation based dictionary learning~\cite{kang2012automatic} (denoted as SR), discriminative sparse coding~\cite{luo2015removing} (denoted as DSC) and layer prior based method~\cite{li2016rain} (denoted as LP).
We first collect 5 images from \cite{li2016rain} and follow \cite{garg2006photorealistic} and \cite{he2011single} to generate rain streaks and hazes as our synthesized testing images. It can be observed in Figure~\ref{fig:visual} that SR approach tends to over-smooth the background and thus degrade the image quality. Though LP can perform well on rain streaks removal, the hazes still exist in the result images. In contrast, by simultaneously performing rain streaks and hazes removal, our proposed method can achieve the best performance among all compared approaches.
We then evaluate all these methods on 4 real-world heavy rain images from \cite{luo2015removing}. Again, we obtain the best qualitative performance among all the compared approaches.


\section{Conclusions}

This paper developed DPATN for single image dehazing. Different from previous works, which
either design models based on priors or
build networks in heuristic ways,
DPATN performed transmission propagation by aggregating priors and data in a deep residual architecture.
We investigated its propagation behaviors based on an energy perspective.
A lightweight training framework was also developed for DPATN. Finally, we proposed a task-aware image separation technique to extend DPATN for more challenging vision tasks, such as underwater image enhancement and single image rain removal. The global convergence property also be proved for the our propagation.
Experiments demonstrated that DPATN achieved favorable performance against state-of-the-art
approaches.

\bibliographystyle{IEEEtran}
\bibliography{egbib}
\begin{IEEEbiography}[{\includegraphics[width=1in,height=1.25in,clip,keepaspectratio]{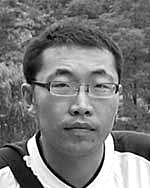}}]{Risheng Liu}
 received the BSc and PhD degrees
both in mathematics from the Dalian University
of Technology in 2007 and 2012, respectively. He
was a visiting scholar in the Robotic Institute of
Carnegie Mellon University from 2010 to 2012. He
served as Hong Kong Scholar Research Fellow at
the Hong Kong Polytechnic University from 2016
to 2017. He is currently an Associate Professor with
the International School of Information and Software
Technology, Dalian University of Technology. His
research interests include machine learning, optimization,
computer vision and multimedia. He was a co-recipient of the IEEE
ICME Best Student Paper Award in both 2014 and 2015. Two papers ware
also selected as Finalist of the Best Paper Award in ICME 2017.
\end{IEEEbiography}
\vspace{-40pt}
\begin{IEEEbiography}[{\includegraphics[width=1in,height=1.25in,clip,keepaspectratio]{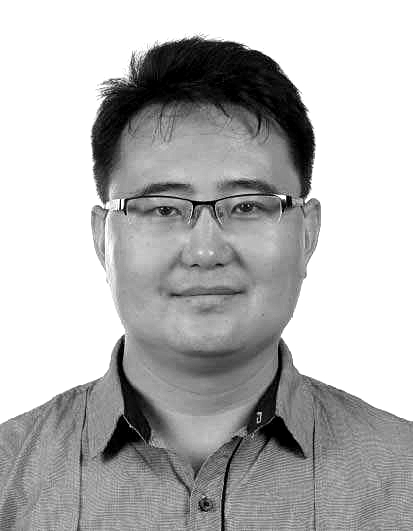}}]{Xin Fan}
 was born in 1977. He received the B.E. and Ph.D. degrees in information and communication engineering from Xian Jiaotong University, Xian, China, in 1998 and 2004, respectively. He was with Oklahoma State University, Stillwater, from 2006 to 2007, as a post-doctoral research Fellow. He joined the School of Software, Dalian University of Technology, Dalian, China, in 2009. His current research interests include computational geometry and machine learning, and their applications to low-level image processing and DTI-MR image analysis.
\end{IEEEbiography}
\vspace{-40pt}
\begin{IEEEbiography}[{\includegraphics[width=1in,height=1.25in,clip,keepaspectratio]{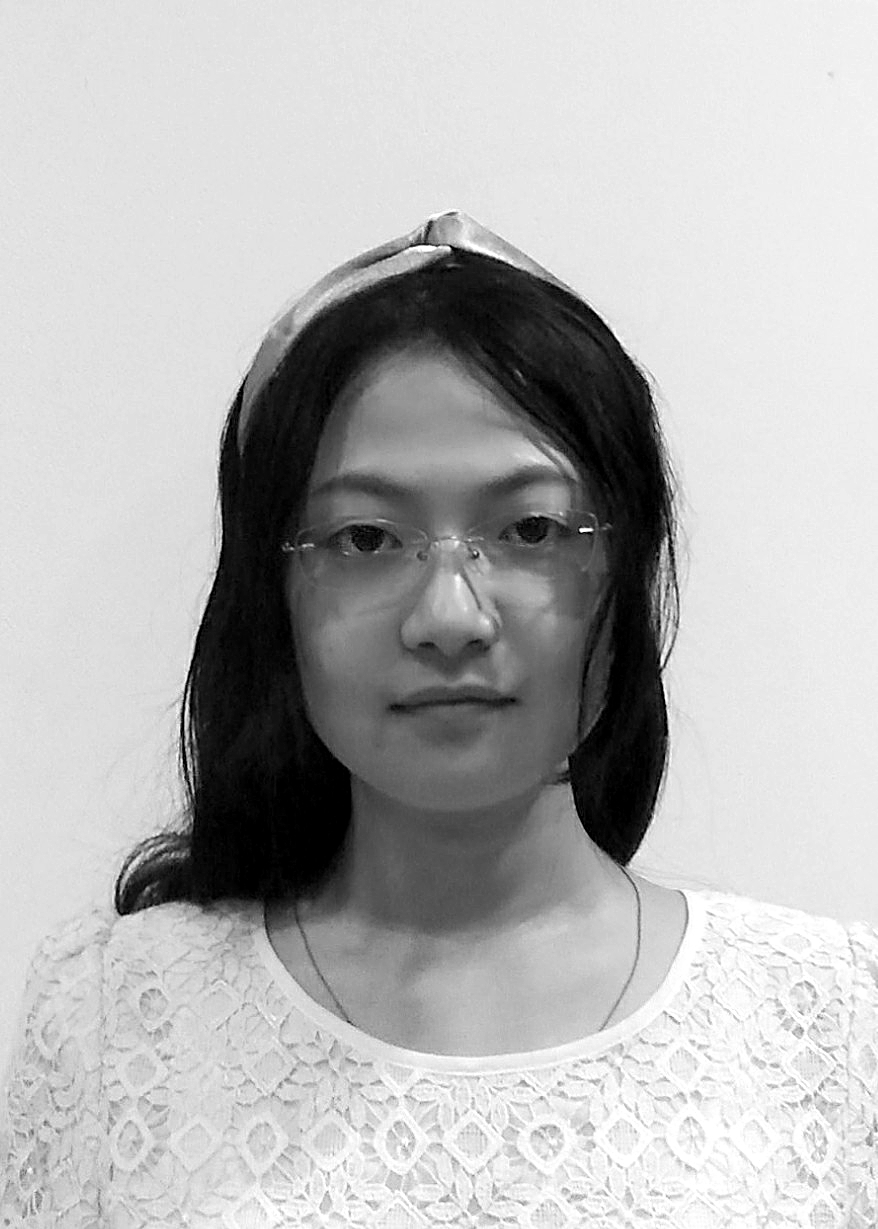}}]{Minjun Hou}
 was born in 1995. She received the B.E. degree in Software Engineering from Dalian University of Technology, China, in 2017. Now she is studying as a postgraduate student in Dalian University of Technology. Her research interests include computer vision, image processing and machine learning.
\end{IEEEbiography}
\vspace{-40pt}
\begin{IEEEbiography}[{\includegraphics[width=1in,height=1.25in,clip,keepaspectratio]{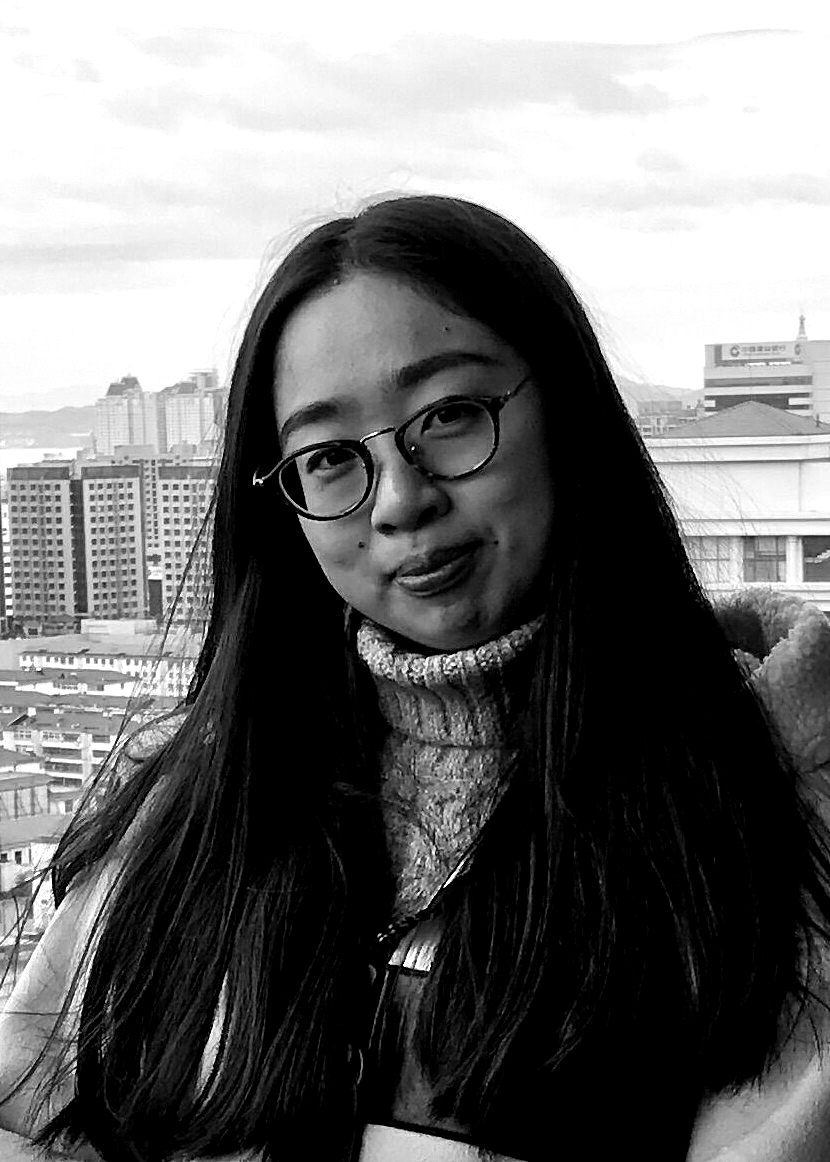}}]{Zhiying Jiang}
was born in 1995. She received the B.E. degree in Software Engineering from Dalian Maritime University, China, in 2017. Now she is studying as a postgraduate student in Dalian University of Technology. Her research interests include computer vision, image processing and machine learning.
\end{IEEEbiography}
\vspace{-40pt}
\begin{IEEEbiography}[{\includegraphics[width=1in,height=1.25in,clip,keepaspectratio]{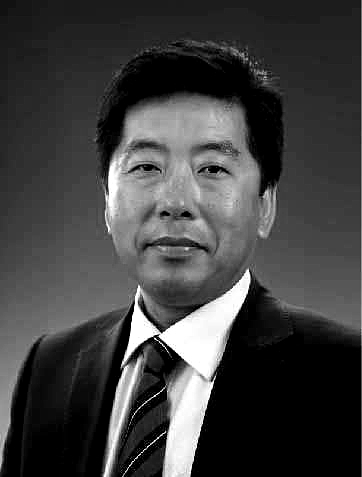}}]{Zhongxuan Luo}
 received the B.S. degree in Computational Mathematics from Jilin University, China, in 1985, the M.S. degree in Computational Mathematics from Jilin University in 1988, and the PhD degree in Computational Mathematics from Dalian University of Technology, China, in 1991. He has been a full professor of the School of Mathematical Sciences at Dalian University of Technology since 1997. His research interests include computational geometry and computer vision.
\end{IEEEbiography}
\vspace{-40pt}
\begin{IEEEbiography}[{\includegraphics[width=1in,height=1.25in,clip,keepaspectratio]{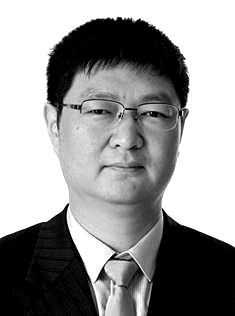}}]{Lei Zhang}
 received his B.Sc. degree in 1995 from Shenyang Institute of Aeronautical Engineering, Shenyang, P.R. China, and M.Sc. and Ph.D degrees in Control Theory and Engineering from Northwestern Polytechnical University, Xi’an, P.R. China, respectively in 1998 and 2001, respectively. From January 2003 to January 2006 he worked as a Postdoctoral Fellow in the Department of Electrical and Computer Engineering, McMaster University, Canada. He has been a Chair Professor, Since 2017. His research interests include computer vision, pattern recognition, image and video analysis, and biometrics.
\end{IEEEbiography}

\end{document}